\documentclass[11pt]{article}

\usepackage[preprint]{acl}

\usepackage{times}
\usepackage{latexsym}

\usepackage[T1]{fontenc}

\usepackage[utf8]{inputenc}

\usepackage{microtype}

\usepackage{inconsolata}

\usepackage{graphicx}
\usepackage{bbm}

\usepackage{wrapfig}
\usepackage{multirow}
\usepackage{float}

\usepackage[utf8]{inputenc} 
\usepackage[T1]{fontenc}    
\usepackage{url}            
\usepackage{booktabs}       
\usepackage{amsfonts}       
\usepackage{nicefrac}       
\usepackage{microtype}      
\usepackage{xcolor}         

\usepackage{amsthm,amsmath,amssymb}
\usepackage{bbm}
\usepackage{graphicx}
\usepackage{float}
\usepackage[subrefformat=parens,farskip=0pt]{subfig}
\usepackage{enumitem}

\usepackage{wrapfig}
\usepackage{multirow}

\newtheorem{theorem}{Theorem}
\newtheorem{lemma}{Lemma}

\newcommand{\eg}{\textit{e}.\textit{g}.}
\newcommand{\ie}{\textit{i}.\textit{e}.}

\newcommand{\lh}[1]{{\color{black}{#1}}}
\newcommand{\gjy}[1]{{\color{black}{#1}}}
\title{Debiased Orthogonal Boundary-Driven Efficient Noise Mitigation}

\author{
 \textbf{Hao Li\textsuperscript{1}\thanks{These authors contribute equally to this work.}},
 \textbf{Jiayang Gu\textsuperscript{2}$^*$},
 \textbf{Jingkuan Song\textsuperscript{3}\thanks{Correspondence.}},
 \textbf{An Zhang\textsuperscript{4}},
 \textbf{Lianli Gao\textsuperscript{5}}
\\
\\
 \textsuperscript{1}Washington University in St. Louis,
 \textsuperscript{2}University of Warwick, \textsuperscript{3}Tongji University, \\
 \textsuperscript{4}University of Science and Technology of China, \\
 \textsuperscript{5}University of Electronic Science and Technology of China
  \\
   \small\texttt{\{18th.leolee, jiayang.barrygu, jingkuan.song\}{@gmail.com}}
}

\begin{document}

\maketitle
\begin{abstract}
\label{sec:abstract}
Mitigating the detrimental effects of noisy labels on the training process has become increasingly critical, as obtaining entirely clean or human-annotated samples for large-scale pre-training tasks is often impractical. Nonetheless, existing noise mitigation methods often encounter limitations in practical applications due to their task-specific design, model dependency, and significant computational overhead. In this work, we exploit the properties of high-dimensional orthogonality to identify a robust and effective boundary in cone space for separating clean and noisy samples. Building on this, we propose One-Step Anti-noise (OSA), a model-agnostic noisy label mitigation paradigm that employs an estimator model and a scoring function to assess the noise level of input pairs through just one-step inference. We empirically validate the superiority of OSA, demonstrating its enhanced training robustness, improved task transferability, streamlined deployment, and reduced computational overhead across diverse benchmarks, models, and tasks. Our code is released at 
\url{https://github.com/leolee99/OSA}
\end{abstract}
\section{Introduction}
\label{sec:intro}

\gjy{

Increasing the scale of data is a relatively efficient way to improve model performance in many domains~\cite{scalinglaw1}. 
As existing works tend to harvest data from the internet by query matching~\cite{cc3m}, structural attribute matching~\cite{laion}, it inevitably introduces label noise into the training process. 
The incompatibility between the data and the label hinders positive samples to learn correct embeddings, thus wasting the distributional diversity brought by large-scale data and the cost of collecting them.
This poses a substantial challenge for robust model training in various tasks, such as cross-modal matching~\cite{NCR,NPC}, image classification~\cite{coteaching+}, and image retrieval~\cite{retrieval_prism}.
}

\lh{Noise mitigation methods aim to eliminate the detriment effects from noisy samples during model training.}
However, existing noise mitigation approaches encounter several limitations that constrain their practical applicability: 1) \textbf{Task specificity:} Existing methods~\cite{NCR, webFG, retrieval_tisnt} are tailored to specific tasks, limiting their applicability across different tasks. 2) \textbf{Model dependency:} Most noise mitigation techniques~\cite{retrieval_prism, BiCro} are tightly coupled with specific models, requiring extensive modifications for adaptation to different models. 3) \textbf{Computational cost:} Numerous existing methods necessitate dual-model collaborations~\cite{NCR, coteaching+} or multiple training passes~\cite{NCR}, \ie, they require at least two backward passes per training step, effectively doubling the computational expense and substantially increasing the training burden (see Figure.~\ref{fig1:cost}). Motivated by the remarkable generalization capabilities of multimodal pre-trained models such as CLIP~\cite{CLIP}, several studies~\cite{CLIPCleaner, NoisyDetector, NPC, NoisyDetector2} start leveraging these pre-trained models for noise mitigation. However, these approaches still suffer from aforementioned limitations, including task specificity~\cite{CLIPCleaner, NoisyDetector, NoisyDetector2}, model dependency~\cite{NoisyDetector}, and excessive computational demands~\cite{NPC}, making them hard to utilize in practical scenarios.

\begin{figure*}[t]
    \centering
    \small 
    \subfloat[Multiple backwards enhancing cost]{\label{fig1:cost}
        \includegraphics[width=0.4\linewidth,trim=225 225 225 225,clip]{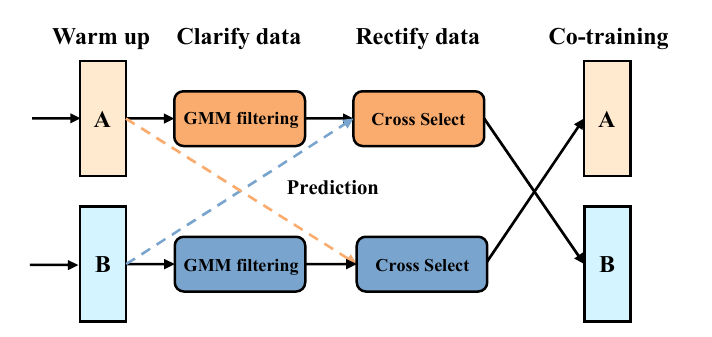}
    }
    \subfloat[Task paradigm unification]{\label{fig1:task}
        \includegraphics[width=0.38\linewidth]{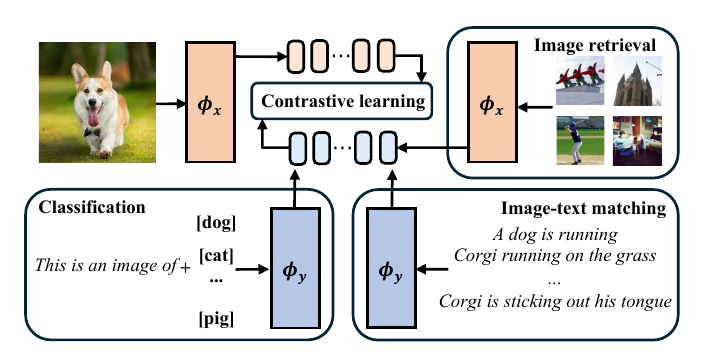}
    }\\
    \subfloat[CLIP on COCO]{\label{fig1:a}
        \includegraphics[width=0.23\linewidth,trim=0 0 0 10,clip]{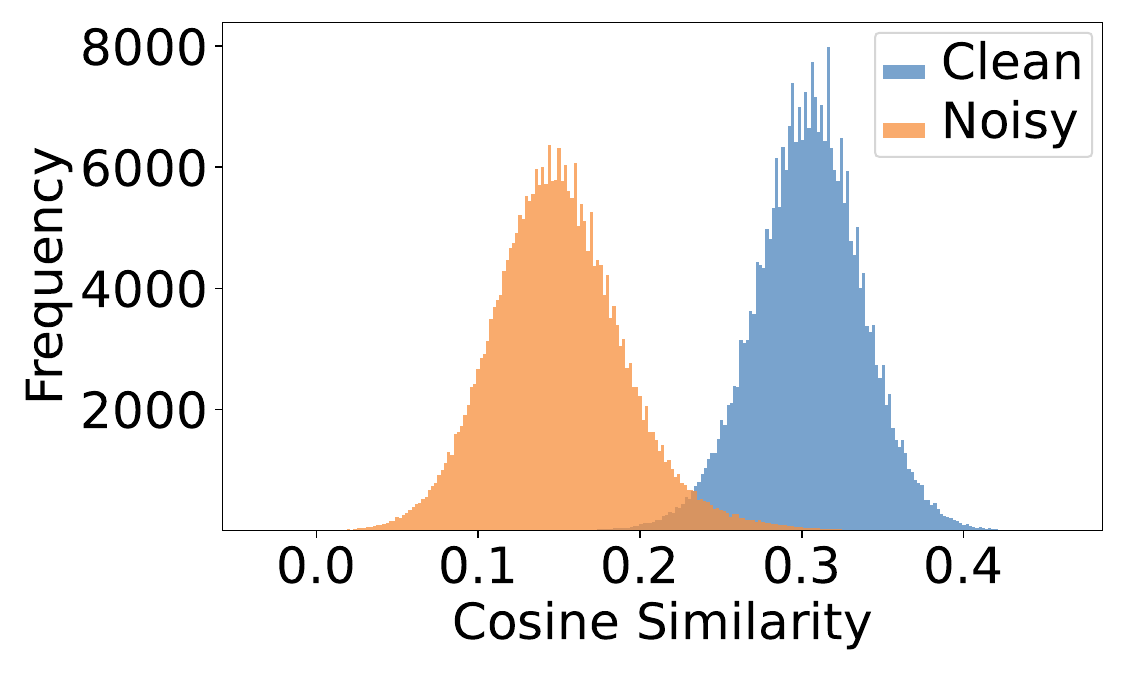}
    }
    \subfloat[ALIGN on COCO]
    {\label{fig1:b}
        \includegraphics[width=0.23\linewidth,trim=0 0 0 10,clip]{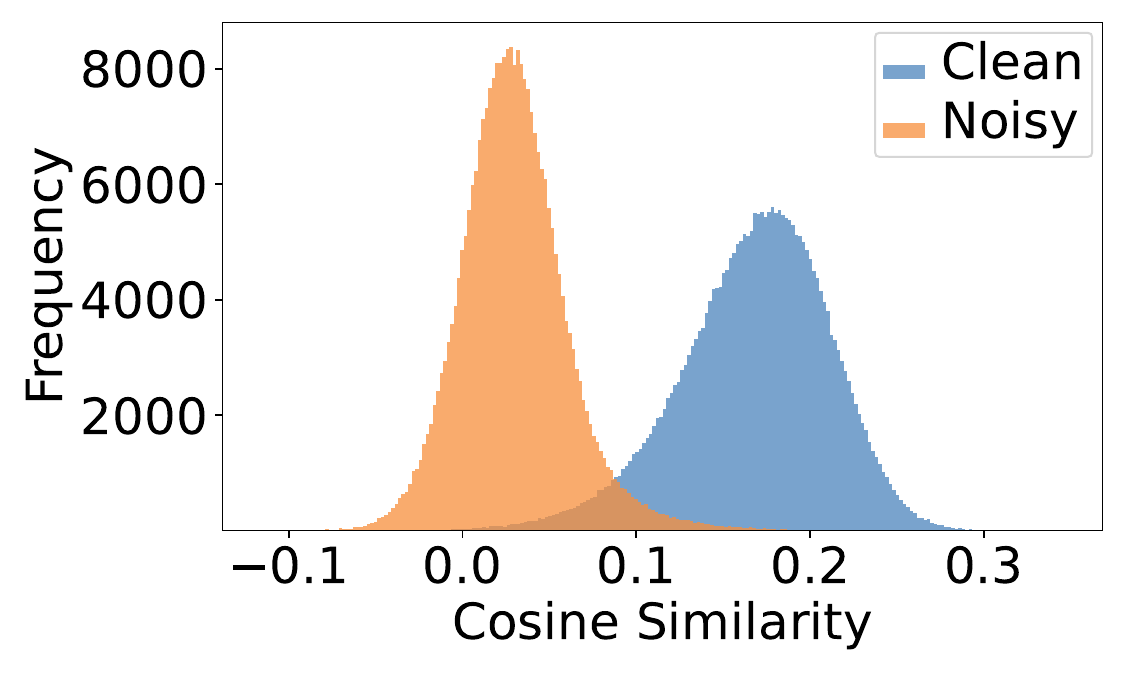}
    }
    \subfloat[CLIP on SDM]{\label{fig1:c}
        \includegraphics[width=0.23\linewidth,trim=0 0 0 10,clip]{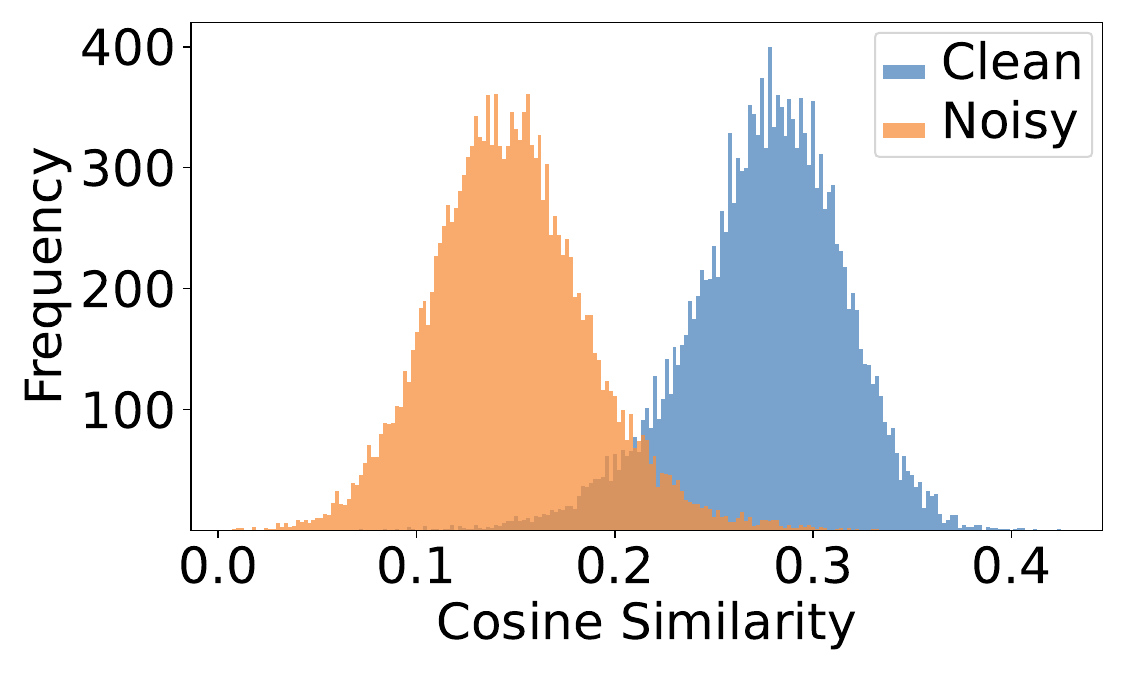}
    }
    \subfloat[ALIGN on SDM]{\label{fig1:d}
        \includegraphics[width=0.23\linewidth,trim=0 0 0 10,clip]{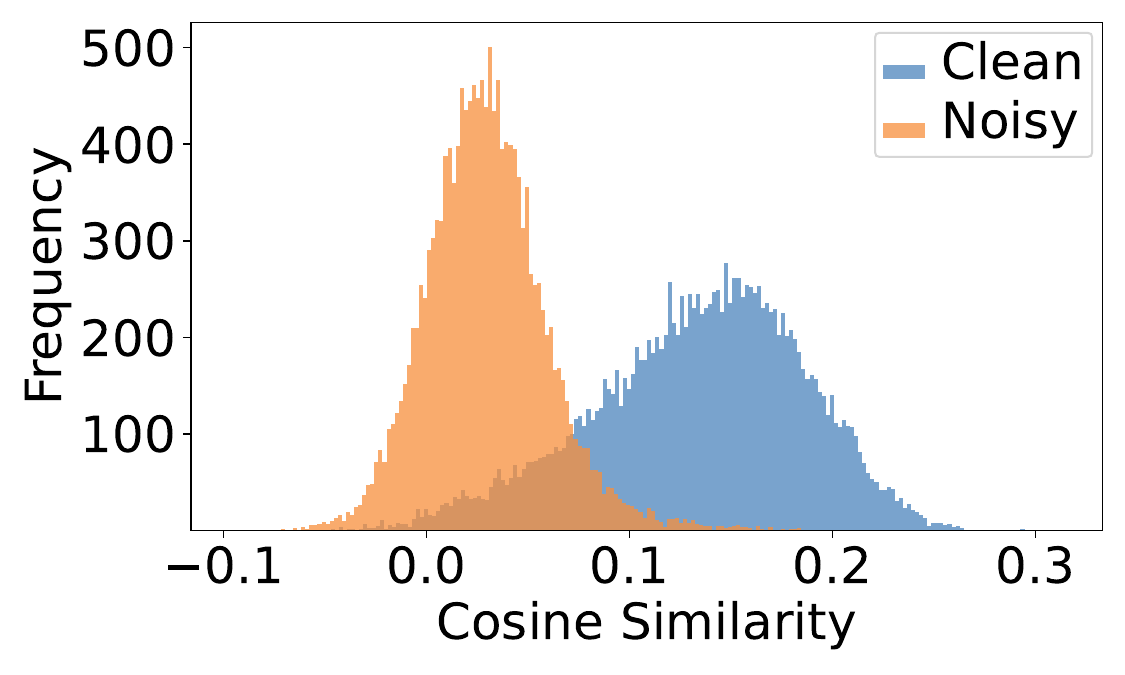}
    }
\caption{\textbf{(a)} The current anti-noise paradigm with multiple backward significantly enhances the training overhead. \textbf{(b)} CLIP unifies the framework of image-text matching and image classification through a shared space. \textbf{(c-f)} Cosine similarity distribution of noise and clean data with 50\% noise.} 
\label{fig:introduction}
\vspace{-5mm}
\end{figure*}

To tackle these challenges, we use an external estimator to assess the noise level of each sample, ensuring the target model-agnostic. 
This estimator reduces the influence of noisy samples by reducing their weights of training loss closer to zero.
We leverage multimodal pre-trained models as the estimator due to their revealed strong semantic capabilities and task transferability.
For instance, CLIP~\cite{CLIP} unifies the paradigms of image-text retrieval and image classification through a shared embedding space (see Figure.~\ref{fig1:task}). It converts category labels into sentences and then calculates the cosine similarity with the image representation to perform image classification. 
In this case, only one additional inference process is required for each sample, significantly reducing the computational overhead compared to performing an extra backward pass. 


Nonetheless, this paradigm introduces a new challenge: how to accurately identify noise based solely on cosine similarity scores inferred by estimators. An ideal solution is to find a decision boundary in cosine space that can separate clean and noisy samples. Existing methods~\citep{NCR, dividemix, CLIPCleaner, NPC} typically attempt to build this boundary within the loss space, an isotropic space with uniform distribution, which creates only a narrow gap between noisy and clean samples. More critically, the coarse handling of overlaps by integrating multi-model predictions often results in an unstable decision boundary. In contrast, the shared embedding space of pre-trained models is a high-dimensional space, \lh{and its corresponding cosine similarity space is an} anisotropic space with the non-uniform distribution. Thus, a consideration is whether the properties of imbalanced anisotropic space can help to identify a more precise and robust decision boundary.


In this work, we delve into the issue of decision boundary selection in anisotropic cosine spaces for pre-trained models to be efficient noise estimators.
Theoretically, a cosine similarity of zero—\ie, an orthogonal boundary—should serve as a natural decision threshold in anisotropic cosine spaces to separate clean and noisy samples. To validate this hypothesis, we empirically analyze the cosine similarity distributions of clean and noisy samples using multimodal pre-trained models CLIP~\cite{CLIP} and ALIGN~\cite{ALIGN} across two datasets: MSCOCO and SDM (Stable Diffusion Model~\cite{StableDiff}), both with a 50\% noise ratio. The SDM dataset, comprising images generated in uncommon artistic styles (see Figure.~\ref{fig:SDM}), is designed to test the robustness of pre-trained models in distinguishing noisy samples from unseen domains.

Surprisingly, as shown in Figure~\ref{fig1:a}-\ref{fig1:d}, the empirically optimal decision boundary deviates significantly from the theoretical orthogonal threshold zero, limiting its usage in practical applications. Despite this deviation, there are also two interesting observations:
\textbf{(1)} the intersection points of clean and noisy distributions remain consistent for the same model across different datasets, suggesting the existence of a stable,
dataset-irrelevant boundary.
\textbf{(2)} even on the SDM dataset—where models encounter unfamiliar domains—the overlap between clean and noisy distributions remains minimal, indicating the boundary’s robustness in distinguishing noisy samples.



\lh{Building on these two observations, we aim to reveal the underlying mechanisms and provide a theoretical guidance for fully exploiting the potential of pre-trained models in noise mitigation. Our key contributions are as follows:}

1. We theoretically figure out the origin of the intersection, attributing it to the shift of orthogonal boundaries induced by the cone effect and elaborate the stability and accuracy of this boundary in separating noisy and clean samples. 


\lh{2. Building on our findings, we develop One-Step Anti-noise (OSA), an efficient and model-agnostic paradigm for noise recognition that requires just one-step inference. 
}

3. We conduct comprehensive experiments across a variety of challenging benchmarks, models, and tasks, demonstrating the effectiveness, generalization capabilities, and efficiency of our findings and introduced methods.

\section{Boundary Principle Analysis}
\label{sec2:analysis}

In Figure~\ref{fig1:a}-\ref{fig1:d}, we observe a natural boundary emerging in the pre-trained model’s ability to distinguish between clean and noisy samples. In this section, we explain the principle of boundary forming from high-dimensional perspectives, and how robust it is in general noise mitigation.

\subsection{Hypothesis: Intersection Boundary is Shifted from Orthogonal Boundary}
\label{sec:hypothesis}
We first analyze the margin between the positive and negative regions defined by the orthogonal decision boundary. We then provide the rationale for our hypothesis that the intersection boundary in Figure~\ref{fig:introduction} can be interpreted as a shifted orthogonal boundary within the conic space.

\paragraph{The orthogonal boundary clearly partitions the positive and negative regions.}
\label{sec:boundarygap}
High-dimensional orthogonality is a well-known phenomenon often referred to as the curse of dimensionality, where the angles between randomly selected vectors tend to concentrate around $90^\circ$, implying cosine similarities close to zero. For example, in a 1024-dimensional space, the probability that two random vectors have a cosine similarity within $[-0.1, 0.1]$ is approximately 99.86\% (see Appendix~\ref{sec:ortho_proof}). This concentration effect induces a natural boundary at zero cosine similarity, which separates vectors with positive and negative similarity values with a wide margin.
\begin{table}[h]
    \renewcommand{\arraystretch}{1}
    \setlength{\tabcolsep}{5mm}
    \centering
    \caption{The mean and variance of cosine similarity between randomly generated pairs.}
    \scalebox{0.8}{
    \begin{tabular}{c|cc}
    \toprule
    Model & Mean & Var\\
    \cmidrule(r){1-3}
    CLIP & 0.215 &  0.024 \\
    ALIGN & 0.087 & 6e-4 \\
    \bottomrule
    \end{tabular}}
    \vspace{-0.2cm}
    \label{table:mean_var}
\end{table}
\paragraph{Cone effect may induce orthogonal boundary shift.}
Recent literature~\cite{mindthegap,querybank,conebert} has demonstrated that the cone effect is a general phenomenon in deep neural networks, where the learned embedding subspace forms a narrow cone and the orthogonal boundary encounters a positive shift. Based on this, a hypothesis is that the intersection boundary in Figure~\ref{fig:introduction} is the shifted orthogonal boundary. To prove this, we simulate the process of selecting random vectors in high-dimensional space and randomly generate thousands of pairs mapped into the shared embedding space. We find that all similarity of these random vector pairs tends to a fixed value, with the low-variance cosine similarity almost lying in the middle of clean and noise distributions (see Table.~\ref{table:mean_var}). 
An interesting phenomenon is that if we compare the mean with the intersection points in Figure~\ref{fig1:a}-\ref{fig1:d}, we find they are almost identical, suggesting that the intersection is highly likely to be a shifted orthogonal boundary in cone space.

\subsection{Theoretical Verification of Intersection Origin}
\label{sec:cone effect}

Here, we theoretically examine whether the intersection boundary originates from a shifted orthogonal boundary. Specifically, we demonstrate that (i) the relative ordering of pairwise cosine similarities remains invariant under projection into the conic space, and (ii) contrastive learning separates clean and noisy samples into opposite half-spaces defined by the orthogonal boundary. From (i) and (ii), it follows that the intersection boundary between the noise-free and noisy distributions corresponds to the shifted orthogonal boundary.

\paragraph{Relative relationship unchanged in propagation process.} 
We study how the boundary shifts from the entire space to the narrow cone in the neural network. The following theorem shows that the cosine similarity will be proportionally scaled to the target narrow cone, while still maintaining a boundary with properties similar to the orthogonal boundary. In other words, vectors with cosine similarity smaller than the orthogonal boundary in the original space remain smaller than the shifted boundary in the narrow cone space, while those larger remain larger.

\begin{theorem}[Proportional shift of boundary]\label{thm:1}
Let $\mathbbm{R}^{d_{in}}$ be the original space before being transmitted in a neural network. Suppose $u,v \in \mathbbm{R}^{d_{in}}$ are any two random vectors with $\cos(u,v) \approx 0$. $u_c,v_c \in \mathbbm{R}^{d_{in}}$ is a pair of clean vectors with $\cos(u_c,v_c) > 0$, while $u_n,v_n \in \mathbbm{R}^{d_{in}}$ is a noisy pair with $\cos(u_n,v_n) < 0$. Given a Neural Network $F(x) = {f_t(f_{t-1}(\dots f_2(f_1(x))))} \in \mathbbm{R}^{d_{out}}$ with $t$ layers. $f_i(x)=\sigma_i(\mathbf{W}_ix+\mathbf{b}_i)$ denotes $i^{th}$ layer, where $\sigma(\cdot)$ indicates activation function. $\mathbf{W}_i \in \mathbbm{R}^{d^i_{out} \times d^i_{in}}$ is a random weight matrix where each element $\mathbf{W}_i^{k,l} \sim \mathcal N(0,1/d_{out}^{i})$ for $k \in \left [ d^i_{out}  \right ] $, $l \in \left [ d^i_{in}  \right ] $, and $\mathbf{b}_i \in \mathbbm{R}^{d^i_{out}}$ is a random bias vector such that $\mathbf {b}^k_i \sim \mathcal N(0,1/d^{i}_{out})$ for $k \in \left [ d^i_{out}  \right ]$. Then, there always be a boundary $\beta$, satisfying:
\begin{equation}
\begin{aligned}
    \cos(F(u_n), F(v_n)) &< \cos(F(u), F(v)) \\
    &\quad\approx \beta < \cos(F(u_c), F(v_c)).
\end{aligned}
\end{equation}

\end{theorem}
Theorem.~\ref{thm:1} shows the pairs' relative relationship in the original entire space remain unchanged after propagating to the narrow cone space of the trained model, and there is always a boundary $\beta$ concentrated on most random vectors. Appendix.~\ref{sec:proof1} provides a detailed statement and proof.

\paragraph{Contrastive learning empowers the separation of clean and noisy samples.}
For an initialized model to learn an embedding space, both clean and noisy samples are treated as orthogonal random vectors since lacking semantic perception ability in the initial space. During contrastive training process, given $N$ sample pairs $\{(x_i,y_i)\}^N_{i=1}$, the embedding space is optimized through the cross-entropy loss: 
\begin{equation}
\mathcal L_{ce} = \frac{1}{N}\sum^N_{i=1}\log\frac{\exp(m_{ii})}{\sum^N_{j=1}\exp(m_{ij})},
\label{eq:celoss}
\end{equation}
where $M \in \mathbbm{R}^{N\times N}$ represents the cosine similarity matrix of $N$ sample pairs during training process. Each element $m_{ij} \in M$ denote the cosine similarity between $x_i$ and $y_j$. The diagonal elements $m_{ii}$ denote the cosine similarities of positive pairs, while the non-diagonal elements $m_{ij}$ represent the cosine similarities of negative pairs.

To minimize $\mathcal L_{ce}$ during training, two subprocesses occur: the diagonal elements of the matrix (\ie, clean pairs) are optimized to the positive side of the orthogonal boundary, while the non-diagonal elements (equivalent to noise pairs) are optimized to the negative side. Consequently, the distributions of these two types of samples are on opposite sides of the orthogonal boundary. Combining with the Theorem.~\ref{thm:1}, we can determine that the interaction boundary between clean and noise in post-trained model is the shifted of orthogonal boundary.



\subsection{Discussion on Applicability} 

\paragraph{How about the boundary robustness even in unfamiliar domains?}
\label{sec:robustness}
\gjy{
To further explore the boundary's generalization ability, especially fine-tuning in unfamiliar domains, we conduct a qualitative analysis on how our boundary can still be effective in this more practical scenario.
}
Given a positive pair from an unseen domain, due to the contrastive learning process during pre-training, it still has a strong likelihood of moving toward the positive side of the boundary, while the negative pair tends toward the negative side. Although the cosine similarity difference might be slight, as we have shown in Section.~\ref{sec:boundarygap}, the boundary constructs a significant gap from the perspective of high-dimensional orthogonality.

\paragraph{How to handle the overlaps through imbalanced probability?}
\label{sec:overlap}
Since orthogonal boundary properties, as cosine similarity decreases and approaches zero from the positive side, the probability of positive samples sharply decreases. Therefore, we can design a scoring function to annotate the cleanliness of samples. This function should satisfy two requirements: for samples with cosine similarity less than or equal to zero, which are almost certainly noise, the function should assign them a weight of zero. For samples with cosine similarity greater than zero, the function gradient should increase rapidly as the cosine similarity moves further from zero. 

\section{Method}
\label{sec:method}

In this section, we present our One-Step Anti-noise (OSA) paradigm with a workflow shown in Figure.~\ref{fig:Framework}. We first define the pair-based noise mitigation tasks in Sec.~\ref{sec:task}. Afterward, we clarify OSA in Sec.~\ref{sec:osd}.

\subsection{Task Definition}
\label{sec:task}
Let $\mathcal D=\{(x_i, y_i, c_i)\}_{i=1}^N$ denote a paired dataset, where $(x_i, y_i)$ represents the $i$-th pair in the dataset, and $c_i$ indicates a noise label for that pair. Specifically, when $c_i = 0$, $(x_i, y_i)$ forms a correct (paired) match, while $c_i = 1$ denotes an incorrect (unpaired) match. The objective of noise mitigation in contrastive learning is to construct a shared embedding space that brings $x_i$ and $y_i$ closer when $c_i=1$. In different tasks, $x_i$ and $y_i$ are distinct data types. For instance, in the image-text retrieval task, $x_i$ and $y_i$ represent images and texts, respectively. In the image classification task, $x_i$ and $y_i$ represent images and categories, respectively. In the image retrieval task, $x_i$ and $y_i$ represent images and relevant images, respectively. The paired sample $(x,y)$ could be encoded into a shared embedding space by corresponding encoders $\phi_x(\cdot)$ and $\phi_y(\cdot)$. Afterward, the cosine similarity $s(x,y)$ is calculated through Eq.~\ref{eq:cos} as semantic relevance of $(x,y)$ to guide the training.
\begin{equation}
s(x,y) =\frac{\phi _ x (x)}{\left\| \phi _ x (x)\right\|} \cdot \frac{\phi _y (y)}{\left\| \phi _y (y)\right\|}\text{.}
\label{eq:cos}
\end{equation}
\subsection{One-step Anti-Noise}
\label{sec:osd}
The workflow of our noise mitigation approach OSA is depicted in Figure.~\ref{fig:Framework}. Initially, we utilize an estimator model to encode the input pair to a shared embedding space and continue to compute the cosine similarity between the paired embeddings. Afterward, the cosine similarity is converted to a cleanliness score $w_i, (0 \leq w_i \leq 1)$ through a scoring function designed based on orthogonal properties (Section.~\ref{sec:overlap}). 
This score quantifies the clean degree of the sample, the smaller $w_i$ is, the noisier the sample. 

During the target model training phase, this cleanliness score is used as a weight, directly multiplied by the loss of the corresponding sample to facilitate selective learning. This noise mitigation process, being solely dependent on the estimator model, is readily adaptable to the training of various target models by simply adding an extra coefficient to the loss function, ensuring the model-agnostic property. Therefore, the key of our noise mitigation approach revolves around the estimator model and noise score assessment.
\begin{figure}[t]
    \centering    \includegraphics[width=0.9\linewidth]{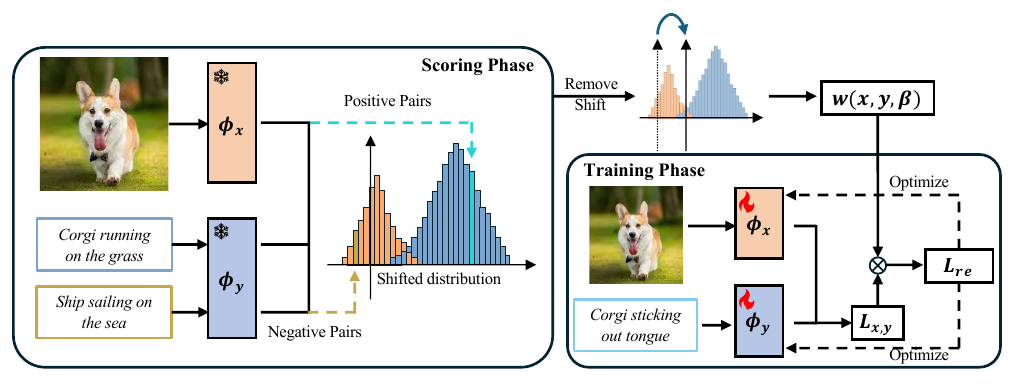}
    \caption{
    \gjy{
    OSA workflow. In the scoring phase, the estimator model calculates the semantic similarity score for input data pairs. Then the score is calibrated by debiasing the effect from the shifted orthogonal boundary. By passing through a scoring function, the in-batch loss is reweighted by the weight $w$.
    }
    }
    \vspace{-0.5cm}
    \label{fig:Framework}
\end{figure}
\vspace{-0.5mm}

\subsubsection{Estimator Model}

\paragraph{Estimator model selection.} In our approach, the Estimator Model must satisfy two critical requirements: 1) effectively mapping input pairs into a unified embedding space and 2) possessing basic semantic understanding capabilities. To meet these requirements, we employ CLIP~\citep{CLIP}, a commonly used multimodal pre-trained models, as our estimator model. It is equipped with a text encoder $\phi_t(\cdot)$ and an image encoder $\phi_v(\cdot)$, enabling it to perform basic zero-shot tasks efficiently.

\paragraph{Domain adaptation (Optional).} While we have performed a qualitative analysis of the zero-shot pre-trained model's robustness on out-of-domain data in Section.~\ref{sec:robustness}, and shown strong robustness for edge cases in Figure.~\ref{fig:introduction}, considering the domain diversity in real-world scenarios, we provide an optional Domain Adaptation (DA) approach to enhance the estimator model's adaptability when encountering edge domains. Following NPC~\citep{NPC}, we first employ a Gaussian Mixture Model (GMM) coupled with strict selection thresholds to ensure the absolute cleanliness of the chosen samples. We afterward implement a warm-up phase with few steps, allowing the estimator model to better understand the semantics of the target domain. Notably, this trick is only optional for our methods. Through multiple experiments, we found that even without domain adaptation, the zero-shot CLIP model performs exceptionally well across various scenarios.

\subsubsection{Noise Score Assessment}

\paragraph{Spatial Debiasing.} The cone effect phenomenon has been demonstrated as a general phenomenon for deep neural networks, typically resulting in a narrow embedding space that causes a shift of space center to a narrow cone center~\citep{mindthegap}. Specifically, when paired randomly generated inputs are mapped into a shared embedding space through model encoders, the resultant vectors exhibit an average cosine similarity that deviates from zero and tends to another fixed angle. To counteract this shift and mitigate its impact on the estimator’s ability to accurately recognize noises through high-dimensional orthogonality, a random sampling method is developed. We begin by constructing $K$ random sample pairs $\mathcal R = \{(x_j,y_j) \mid j=1,2,\dots, K\}$ and processing them through the estimator's encoder to generate a set of vectors. Then the average cosine similarity among these vectors will be calculated as the space shift $\beta$ through:
\begin{equation}
\beta=\frac{\sum^K_{j=1}s(x_j,y_j)}{K}.
\label{eq:relu}
\end{equation}


\paragraph{Scoring Function.}
\label{sec: scoring_main} Following spatial debiasing, we introduce a scoring function $w(\cdot)$ to evaluate the quality of each input pair $(x, y)$. This approach builds upon the orthogonal boundary property detailed in section. ~\ref{sec:overlap}, and utilizes an estimator model trained with contrastive learning. The estimator maps clean pairs to positive similarities and noisy pairs to negative ones, with high-dimensional orthogonality ensuring a robust separation between the two distributions.
The design of $w(\cdot)$ is guided by two principles. First, any pair with a negative cosine similarity $s(x,y)$ is considered noise and assigned a weight of zero. Second, for pairs where $s(x,y)$ surpasses the orthogonal boundary $\beta$, the assigned weight increases sharply as the score moves away from $\beta$. To formalize this, we compute a debiased similarity score $\tilde{s}_{x,y} = s(x,y) - \beta$. The final re-weighting function is then defined as follows (further variants are explored in Appendix~\ref{sec:diff_scoring}):
\begin{equation}
w(x,y,\beta)=\left\{\begin{aligned}
& 0 &, & \tilde{s}_{x,y} \le 0\\
& -(\tilde{s}_{x,y})^2(\tilde{s}_{x,y}-1) &, & \textit{otherwise} \hfill\\
\end{aligned}\right.
\label{eq:scoring}
\end{equation}

\paragraph{Re-weight Training.}
After scoring, the target model can selectively learn from the samples by re-weighting the loss. Noise samples with smaller weights will have a reduced impact on model updates and will be effectively mitigated. For a sample $(x,y)$, let $\mathcal L_{x,y}$ denote its loss, the re-computed loss $\mathcal L_{re}$ is defined as:
\begin{equation}
\mathcal L_{re} = w(x,y,\beta) \times \mathcal L_{x,y}\text{.}
\label{eq:reweight}
\end{equation}

\section{Experiments}
\label{sec:expr}

\subsection{Comparisons with State of The Arts}
\label{sec:sota}

In this section, we present experiments on multiple datasets with label noise, demonstrating the effectiveness of our methods. Firstly, we describe the datasets, metrics, and implementation details. 
Then, we report our results on several downstream tasks.
Lastly, we conduct ablation studies to show how each part of our method contributes and examine how these parts interact. The literature involved in our experiments and richer related work are detailed in Appendix.~\ref{sec:relate}.

\begin{table*}[t]
  \caption{Comparison on noisy MS-COCO.}
  \label{table:coco}
  \renewcommand{\arraystretch}{0.8}
  \setlength{\tabcolsep}{1.5mm}
  \centering
  \scalebox{0.7}{
  \begin{tabular}{c|l|ccc|ccc|ccc|ccc}
    \toprule
    \multirow{3}{*}{Noise ratio} & \multirow{3}{*}{Method} & \multicolumn{6}{c|}{MS-COCO 1K} & \multicolumn{6}{c}{MS-COCO 5K} \\
     & & \multicolumn{3}{c}{i2t} & \multicolumn{3}{c|}{t2i}  & \multicolumn{3}{c}{i2t} & \multicolumn{3}{c}{t2i}\\
     &  &  \multicolumn{1}{c}{R@1} & \multicolumn{1}{c}{R@5} & \multicolumn{1}{c}{R@10} &
                                \multicolumn{1}{c}{R@1} & \multicolumn{1}{c}{R@5} & \multicolumn{1}{c|}{R@10} &  \multicolumn{1}{c}{R@1} & \multicolumn{1}{c}{R@5} & \multicolumn{1}{c}{R@10} &
                                \multicolumn{1}{c}{R@1} & \multicolumn{1}{c}{R@5} & \multicolumn{1}{c}{R@10}   
                                \\
    \cmidrule(r){1-14}
     \multirow{7}{*}{0\%}  & VSE$\infty$    &   82.0    &   \textbf{97.2}      &  98.9     & 69.0   & 92.6     & 96.8  &   62.3    &   87.1      &  \textbf{93.3}     & 48.2   & \textbf{76.7}     & 85.5   \\
    &  PCME++  &   81.6    &   \textbf{97.2}      &  99.0     & \textbf{69.2}   & \textbf{92.8}     & 97.1 &   62.1    &   86.8      &  \textbf{93.3}     & 48.1   & \textbf{76.7}     & \textbf{85.5}\\
     & PAU   &  80.4  & 96.2  &  98.5 & 67.7  & 91.8 & 96.6 &   63.6  & 85.2  &  92.2 & 46.8  & 74.4 & 83.7 \\
    & NPC     & \textbf{82.2}  &   96.5  & \textbf{98.7} & 68.3  & 92.0 & \textbf{98.7} &  65.4   & \textbf{87.3} &  93.1 & 48.5  & 75.4 & 84.4 \\     
    \cmidrule(r){2-14}
    & CLIP     &  80.1   &   95.7  &  98.2 & 67.1  & 91.4 & 96.6 &  62.9   &   84.9  &  91.6 & 46.5  & 73.8 & 82.9 \\
    & \textbf{\qquad +OSA}     &  \textbf{82.2}   &   96.5  &  \textbf{98.7} & 68.8  & 92.1 & 96.7 &  \textbf{65.6}  & 86.8 & 92.9 & \textbf{49.1}  & 76.2 & 84.8 \\
    \cmidrule(r){2-14}   
    & ALIGN     &  84.9   &   97.3  &  \textbf{99.0} & 70.5  & 92.8 & 97.2 &  69.6   &   \textbf{89.9}  &  94.5 & 50.5  & 77.5 & 85.7 \\ 
    & \textbf{\qquad +OSA}     &  \textbf{85.3}   &   \textbf{97.4}  &  \textbf{99.0} & \textbf{71.4}  & \textbf{93.1} & 97.3 &   \textbf{69.8}   &   \textbf{89.9}  &  \textbf{94.8} & \textbf{51.4}  & \textbf{78.2} & \textbf{86.3} \\
    
    \cmidrule(r){1-14}
     \multirow{7}{*}{20\%}  & VSE$\infty$     &  78.4    &   94.3      &  97.0     & 65.5   & 89.3     & 94.1  &  58.6    &   83.4      &  89.9     & 45.0   & 72.9     & 81.7   \\
     & PCME++    &   78.4    &   95.9      &  98.4     & 64.9   & 90.8     & 96.1  &   57.7    &   83.9      &  91.0     & 43.2   & 72.3     & 82.4   \\
    & PAU  &   78.2  & 95.2  &  98.1 & 64.5  & 90.0 & 95.4 &   59.3  & 82.9  &  90.4 & 44.2  & 71.3 & 81.3 \\
    & NPC     &  79.9   &   95.9  &  98.4 & 66.3  & 90.8 & \textbf{98.4} &  61.6   &   85.4  &  91.6 & 46.0  & 73.4 & 82.9 \\   
    \cmidrule(r){2-14}
    &  CLIP     &   76.0   &   94.3  &  97.5 & 63.4  & 89.0 & 94.8 &   55.3   &   79.1  &  86.9 & 41.0  & 68.8 & 79.3 \\
    & \textbf{\qquad +OSA}     &  \textbf{81.6}   &  \textbf{96.2}  &  \textbf{98.5} & \textbf{68.9}  & \textbf{92.0} & 96.6 &  \textbf{65.8}  & \textbf{86.4} & \textbf{92.5} & \textbf{48.7}  & \textbf{76.1} & \textbf{84.5} \\
    \cmidrule(r){2-14}
    & ALIGN     &   79.4   &   95.7  &  98.2 & 66.2  & 90.8 & 96.1 &   60.9  &   84.5  &  91.0 & 46.3  & 73.6 & 82.3 \\ 
    & \textbf{\qquad +OSA}     &  \textbf{85.1}   &  \textbf{97.4}  &  \textbf{99.1} & \textbf{70.9}  & \textbf{93.0} & 97.3 &  \textbf{69.7}  & \textbf{90.0} & \textbf{94.7} & \textbf{50.9}  & \textbf{77.8} & \textbf{86.2} \\    
    
    \cmidrule(r){1-14}
     \multirow{7}{*}{50\%}  & VSE$\infty$   &   44.3    &   76.1      &  86.9     & 34.0   & 69.2     & 84.5  &   22.4    &   48.2      &  61.1     & 15.8   & 38.8     & 52.1   \\
    & PCME++    &   74.8    &   94.3      &  97.7     & 60.4   & 88.7     & 95.0 &   52.5    &   79.6      &  88.4     & 38.6   & 68.0     & 79.0    \\
    & PAU   &  76.4  & 94.1  &  97.6 & 62.3  & 88.5 & 94.6 &  57.3  & 81.5  &  88.8 & 41.9  & 69.4 & 79.6 \\
    & NPC     &  78.2   &   94.4  &  97.7 & 63.1  & 89.0 & \textbf{97.7}  &  59.9   &   82.9  &  89.7 & 43.0  & 70.2 & 80.0\\
    \cmidrule(r){2-14}
    & CLIP     &  73.9   &   93.0  &  97.2 & 60.1  & 87.3 & 94.0  &  54.1   &   78.5  &  86.6 & 39.7  & 67.2 & 77.5\\
    & \textbf{\qquad +OSA}     &   \textbf{81.4}   &  \textbf{96.5} &  \textbf{98.6} & \textbf{68.4}  & \textbf{92.0} & \textbf{96.6} &  \textbf{64.7}  & \textbf{86.8} & \textbf{92.4} & \textbf{48.6}  & \textbf{75.9} & \textbf{84.6} \\
    \cmidrule(r){2-14}
    & ALIGN     &  78.0   &   95.8  &  98.5 & 65.4  & 90.3 & 96.0  &  60.1   &   84.3  &  91.2 & 45.2  & 72.8 & 82.1 \\ 
    & \textbf{\qquad +OSA}     &  \textbf{84.3}   &   \textbf{97.0}  &  \textbf{98.9} & \textbf{70.0}  & \textbf{92.5} & 97.0 &  \textbf{68.5}  & \textbf{89.2} & \textbf{94.2} & \textbf{50.0}  & \textbf{77.0} & \textbf{85.4} \\
    \bottomrule
  \end{tabular}}
\end{table*}

\subsection{Evaluation Setting}
In this section, we briefly introduce the datasets and evaluation metrics used in the experiments. For more dataset and implementation details, please refer to Appendix.~\ref{sec:appendix_details_of_implementaion}.

\paragraph{Datasets.} We evaluate our method on three downstream tasks with noisy labels, including one multimodal task and two visual tasks. For the cross-modal matching task, we perform experiments on the \underline{MSCOCO}~\citep{mscoco} and \underline{Flickr30K}~\citep{flickr30k} datasets. Following NPC~\citep{NPC}, we further carry out evaluations on a real-world noisy dataset \underline{CC120K} and \underline{CC3M}. For image classification tasks, experiments are conducted under three subsets of \underline{WebFG-496}~\citep{webFG}—Aircraft, Bird, and Car. 
For image retrieval tasks, we conduct experiments on the \underline{CARS98N} dataset under PRISM~\citep{retrieval_prism} setting.

\paragraph{Evaluation Metrics.} For the image-text matching task, the recall value of the top-K retrieved results (R@K) is used. For classification tasks, accuracy serves as the evaluation metric. For the image retrieval task, we use Precision@1 and mAP@R for evaluation.

\subsection{Multi-task Evaluation}
\paragraph{Results on MSCOCO.}
To fairly demonstrate the effectiveness of our method, we compare OSA with various robust learning image-text matching approaches using the same ViT-B/32 CLIP as backbone, including VSE$\infty$~\citep{VSE}, PCME++~\citep{PCME++}, PAU~\citep{PAU}, NPC~\citep{NPC}. Besides, we separately employ OSA on both CLIP~\citep{CLIP} and ALIGN~\citep{ALIGN}. The results in Table.~\ref{table:coco} show that OSA outperforms all previous approaches on all metrics with a huge gap. In the more challenging MS-COCO 5K set with 50\% noise ratio, OSA surpasses the SOTA method NPC in the R@1 for both image-to-text (i2t) and text-to-image (t2i) matching by 8.6\% and 7.0\%, respectively. Another phenomenon is that as the noise ratio increases from 0\% to 50\%,  all other methods encounter severe performance drop, with an averaging drop of 5.05\% for NPC across four R@1 metrics. In contrast, OSA exhibits only a slight decrease of 1.275\%, showcasing the accuracy and robustness of OSA in anti-noise tasks.

\gjy{
Furthermore, to prove that OSA can generalize to other image-text matching tasks, we evaluate our framework on different synthetic and web-crawled datasets in Appendix \ref{sec:append_qualititive}, including \textbf{Flickr30K}, \textbf{CC120K}, and \textbf{CC3M}.
}

\begin{table}[h]
    \centering
    \caption{Results of other image-based tasks.}
    \renewcommand{\arraystretch}{1.2}
    \setlength{\tabcolsep}{1.2mm}
    \scalebox{0.7}{
    \begin{tabular}{c|ccc|cc}
    \toprule
    \multirow{3}{*}{\textbf{Method}} & \multicolumn{3}{c|}{\textbf{Image Classification}} & \multicolumn{2}{c}{\textbf{Image Retrieval}} \\
    \cline{2-6}
     & Aircraft & Bird & Car  & \multirow{2}{*}{Prec.} & \multirow{2}{*}{mAP}\\  
     & Acc & Acc & Acc & &\\ \hline
    Baseline  &   65.44 &  62.29 & 75.90  & 71.69 & 18.16   \\
    \textbf{+OSA}    &   \textbf{73.18} & \textbf{70.50} & \textbf{80.19}     & \textbf{78.45} & \textbf{24.99} \\
    \bottomrule
    \end{tabular}
    }
    \vspace{-0.2cm}
    \label{table:cls_and_retreival}
\end{table}

\paragraph{Results on image-based Tasks.}
We validate the transferability of OSA by evaluating it on two single-modality tasks: image classification and image retrieval. The results are presented in Table.~\ref{table:cls_and_retreival}. The baseline method for both tasks leverages contrastive learning. In the image classification task, OSA outperforms the baseline by 7.74\%, 8.21\%, and 4.28\% on the Aircraft, Bird, and Car subsets, respectively. In the image retrieval task, OSA improves performance by 6.76\% in precision and 6.83\% in mAP. These improvements prove the strong task transferability and generality of OSA.

\subsection{Target Model-Agnostic Analysis}
To prove that OSA is an architecture-agnostic paradigm easily adaptable to various models, we evaluate it across different architectures and apply it to other anti-noise models to demonstrate its generalization in noise mitigation.

\paragraph{Architecture-agnostic Analysis.}
The effectiveness of OSA on Vision Transformer (ViT) has been proven in Section.~\ref{sec:sota}. We further explore the generality of OSA on target models with other architectures. Specifically, we deploy OSA above the VSE++~\citep{VSE++} model with two different architecture types: ResNet-152~\citep{resnet} and VGG-19~\citep{vgg19}. These two architectures are highly sensitive to noise~\citep{NCR}. In this experiment, all estimator models employ zero-shot CLIP, and we utilize the original VSE++ as our baseline. The results in Table.~\ref{table:backbone} indicates that a significant performance degradation emerged for the baseline methods in a noisy setting, while a stable performance is achieved after employing OSA. The stable performance of these two noise-vulnerable architectures fully demonstrates that OSA possesses an architecture-agnostic property.

\gjy{
\paragraph{Adaptability to Other Anti-Noise Models.}
To further prove the applicability and transferability of OSA, we apply OSA to the current state-of-the-art model NPC~\citep{NPC} and test several different models as estimators in Appendix \ref{sec:transfer}.
}

\begin{table}[t]
  \caption{The results of the target model with different architectures on noisy MSCOCO 1K.}
  \label{table:backbone}
  \renewcommand{\arraystretch}{0.7}
  \setlength{\tabcolsep}{0.5mm}
  \centering
  \scalebox{0.65}{
  \begin{tabular}{c|l|c|ccc|ccc}
    \toprule
    \multirow{2}{*}{Noise ratio} & \multirow{2}{*}{Method} & \multirow{2}{*}{Architecture} & \multicolumn{3}{c}{i2t} & \multicolumn{3}{c}{t2i} \\
     & & & R@1 & R@5 & R@10 & R@1 & R@5 & R@10 \\
    \cmidrule(r){1-9}
     \multirow{4}{*}{0\%}  & Baseline  & \multirow{2}{*}{ResNet-152} & 58.9  & \textbf{86.9} & \textbf{93.8} & 44.2 & 77.9 & \textbf{88.3}  \\
    & \textbf{\qquad+OSA} &  &  \textbf{58.9}  & 86.2 & 93.7 & \textbf{44.3} & 77.9 & 87.9  \\
    \cmidrule(r){2-9}
    & Baseline  & \multirow{2}{*}{VGG-19} & 49.6 & 79.4 & 89.1 & 38.0 & 72.9 & \textbf{84.7}  \\
    & \textbf{\qquad+OSA} &  & \textbf{50.1} &\textbf{80.0} & \textbf{89.3} & \textbf{38.3} & \textbf{73.0} & 84.6  \\

    \cmidrule(r){1-9}
     \multirow{4}{*}{20\%}& Baseline & \multirow{2}{*}{ResNet-152} &  45.8  & 70.3 & 83.7 & 36.1 & 68.4 & 79.7  \\
    & \textbf{\qquad+OSA} &  &  \textbf{58.1}  &  \textbf{86.1} & \textbf{93.2} & \textbf{43.4} & \textbf{76.8} & \textbf{87.2}  \\
    \cmidrule(r){2-9}
    & Baseline  & \multirow{2}{*}{VGG-19} &  33.2  & 67.1 & 81.5 & 25.9 & 58.0 & 71.4  \\
    & \textbf{\qquad+OSA} &  &  \textbf{49.3}  & \textbf{79.1} & \textbf{88.6} & \textbf{37.2} & \textbf{71.9} & \textbf{83.8}  \\
    
    \cmidrule(r){1-9}
     \multirow{4}{*}{50\%} & Baseline & \multirow{2}{*}{ResNet-152} &  28.4  & 61.2 & 75.2 & 5.2 & 14.0 & 19.5  \\
    & \textbf{\qquad+OSA} &  &  \textbf{55.0}  & \textbf{84.0} & \textbf{92.0} & \textbf{40.7} & \textbf{74.7} & \textbf{85.6}  \\
    \cmidrule(r){2-9}
    & Baseline & \multirow{2}{*}{VGG-19} & 2.5 & 9.8 & 16.2 & 0.1 & 0.5 & 1.0  \\
    & \textbf{\qquad+OSA} &  & \textbf{47.1} & \textbf{77.7} & \textbf{87.6} & \textbf{35.7} & \textbf{70.3} & \textbf{82.8}  \\
    \bottomrule
  \end{tabular}}
  \vspace{-5mm}
\end{table}

\subsection{Estimator Model Analysis}
The estimator model is fundamental to OSA's noise mitigation capabilities. We investigated the impact of various estimator models on a CLIP target model by comparing a baseline (without OSA) against configurations using zero-shot CLIP, zero-shot ALIGN, and a domain-adapted CLIP as estimators.
As detailed in Appendix Table~\ref{table:domain}, the results show that using either CLIP or ALIGN as an estimator significantly enhances the target model's performance, highlighting flexibility in the choice of estimator. Notably, the zero-shot CLIP estimator performed comparably to, and at lower noise ratios, even better than its domain-adapted counterpart. This finding demonstrates the exceptional effectiveness of zero-shot CLIP for noise mitigation and suggests that domain adaptation is not always necessary, which simplifies OSA deployment.


\begin{table}[h]
\centering
    \renewcommand{\arraystretch}{1}
    \setlength{\tabcolsep}{1.5mm}
    \caption{ACC and recall of noise detection.}
    \label{table:acc}
\scalebox{0.75}{
\begin{tabular}{llccccc}
\toprule
\multirow{2}{*}{Estimator Type} & 
\multicolumn{2}{c}{20\% noise} & 
\multicolumn{2}{c}{50\% Noise} \\
\cmidrule(lr){2-3} \cmidrule(lr){4-5}
& Acc. & Recall & Acc. & Recall \\
\midrule
CLIP (w/o DA) & 93.88 & 97.49 & 93.91 & 99.35 \\
CLIP (w DA)  & 97.68 & 97.18 & 98.14 & 99.24 \\
\bottomrule
\end{tabular}
}
\vspace{-0.4cm}
\end{table}

\subsection{Noise Assessment Accuracy}
\paragraph{Noise Detection Accuracy Analysis.}
To figure out how accurate OSA is in recognizing noise, we evaluate the accuracy and recall on CLIP without Domain-Adaptation (w/o DA) and CLIP with Domain-Adaptation (w DA) on noisy MSCOCO. We utilize zero as the threshold to roughly divide pairs into noise and clean sets, respectively. Concretely, we classify scores less than or equal to 0 as noise, and scores greater than 0 as clean. The Accuracy means the proportion of the clean pairs correctly classified into the clean set, while the Recall indicates the noisy pairs correctly classified into the noisy set. The results are presented in Table.~\ref{table:acc} indicates the powerful noise recognition capability of OSA. The remarkable performance on CLIP (w/o DA) fully demonstrates the generality of OSA. 
Another notable phenomenon is that all recall scores converge towards 100, suggesting that OSA can almost entirely eliminate the impact of noise on training.

\paragraph{Noise Re-weighting Accuracy Comparison.} 
\gjy{
To understand why our loss re-weighting method outperforms others, we adopt a rank-based approach to show that our method assigns a lower score than NPC to all noise samples in Appendix \ref{sec:append_noise_rank} with the results shown in Table.~\ref{table:weight_rank}.
}

\subsection{Computational Cost Analysis}
\paragraph{Cost in Pre-training.} To evaluate the practicality of OSA in a real-world pre-training scenario, we estimate the additional computational cost for processing 1 billion data points. Using an NVIDIA RTX 3090 with an inference batch size of 4096, 
processing the MS-COCO dataset consisting of 566,435 pairs takes approximately 153 seconds. At this inference rate, processing 1 billion data points would require approximately 75 hours on a single RTX 3090. This cost is negligible for large-scale pre-training, especially with multiple GPUs for parallel inference.

\paragraph{Time Cost Comparison.} 
To evaluate computational efficiency, we benchmarked the training time of our method, OSA, against the CLIP baseline and the state-of-the-art NPC method. While NPC mitigates noise by estimating each sample's negative impact, its approach necessitates double backward passes. In contrast, our method introduces only a negligible training overhead compared to the baseline. As detailed in Appendix Table~\ref{table:time}, OSA requires only one-tenth of the additional training time needed by NPC, demonstrating its superior efficiency and suitability for large-scale, robust training applications.
\section{Conclusion}
\label{sec:conclusion}
This work investigates noise mitigation in large-scale training by analyzing the orthogonal boundary shifts in pre-trained models. We provide a theoretical framework for leveraging these models and introduce OSA, a novel, model-agnostic anti-noise paradigm. OSA is characterized by its task transferability, model adaptability, and minimal computational overhead. By leveraging the properties of high-dimensional orthogonality, we designed a robust decision boundary to effectively distinguish between noisy and clean samples. Our theoretical analysis and comprehensive experiments validate OSA's efficacy, demonstrating state-of-the-art performance in standard anti-noise settings and particular suitability for large-scale applications. To our knowledge, this is the first study to address noise mitigation in this context and propose a generalizable anti-noise framework leveraging the full potential of pre-trained models.


\section*{Limitations}
\label{sec:limitations}
Limited by the significant computational cost of pre-training, it is difficult for us to evaluate in a real pre-training process. Instead, we simulate large-scale pre-training processes to the greatest extent possible, such as evaluating on the real-world noisy dataset CC120K, which shares similar domains with mainstream pre-training datasets like CC4M and CC12M. Exploring the broad domain adaptability of OSA in real pre-training scenarios will be a valuable direction for future work.

\section*{Ethic Statement}

This paper presents work whose goal is to advance the field of robust learning in noisy correspondences. All aspects of this research comply with ethical considerations and standards of research integrity.
\bibliography{custom}

\clearpage

\appendix
\section*{Appendix}

\section{Details of Implementation}
\label{sec:appendix_details_of_implementaion}
\textbf{Dataset Details.} \underline{MSCOCO} is widely used for noisy cross-modal matching, with each image accompanied by five descriptive captions. 
Following the setting of \citet{NCR}, we utilize 113,287 images for training, 5,000 for validation, and 5,000 for testing.
The \underline{Flickr30K} dataset encompasses 31,783 image-text instances, each image paired with five textual annotations. Adhering to the NCR~\citep{NCR}, we use 29,783 images for training and 1,000 images each for validation and testing. 
Regarding noise splits, following the NCR categorization, we conduct experiments at noise ratios of 0\%, 20\%, 40\%, and 60\%. 
\underline{CC120K} is a real-world multimodal noisy dataset collected by~\citet{NPC} from the Internet, with about 3\%-20\% noise ratio. There are 118,851 image-text pairs for training, 1,000 for validation, and 1,000 for testing.

The Aircraft, Bird, and Car we used in the image classification task are three non-overlapping subsets of the \underline{WebFG-496}~\citep{webFG} dataset.
WebFG-496 consists of 53,339 images, totaling 496 subcategories. This dataset is annotated using a webly supervised approach, which leverages resources from web search engines (\eg, Google Image Search Engine, Bing Image Search Engine) to expand the annotated image dataset.

For the image retrieval task, we conduct experiments on the \underline{CARS98N} dataset under PRISM's setting~\citep{retrieval_prism}. We utilize 9,558 car-related images sourced from text-based searches on Pinterest as the training set, and employ the remaining 98 categories from CARS, unsearched on Pinterest, as a clean test set.
The dataset's noise is inherently real-world, with its creators estimating a noise ratio of approximately 50\%.

\textbf{Implementation Details.}
To demonstrate the effectiveness of the OSA, we incorporate an estimator, built around the core of CLIP, and re-weighting operations based on the Estimator's outcomes into numerous downstream tasks. In the principal task of cross-modal image-text retrieval, we employ CLIP with ViT-B/32 as the baseline and target model by default. All experiments are conducted on a single RTX 3090 GPU using the AdamW optimizer. During both training phases, the model is trained for five epochs with a batch size of 256 and 500 warmup steps.

For the image classification task on the WebFG dataset, we align with the field's prevalent models for a fair comparison by employing the ResNet-50 model enhanced by CLIP for feature extraction and the CLIP image encoder as our estimator. Training and testing are executed on single RTX 3090 GPU, with an input image resolution of $448\times448$. The batch size and initial learning rate are specified as 64 and 1e-5, respectively. In the first phase, the estimator is trained with data modeled by a Gaussian Mixture Model (GMM), which considers the classification and matching losses of all training samples, with the GMM probability threshold of 0.95. The classification task leverages the CLIP protocol, where a fixed prompt (``This is a picture of'') is prepended to category texts. 

For the image retrieval task, we use CLIP ViT-B/32 as the baseline, with a batch size set to 128, an initial learning rate of 5e-6, and the number of epochs set to 10. Following the setup of the PRISM~\citep{retrieval_prism}, we set the parameter for sampling positive examples by the random sampler of the dataloader to 4, and adjust the number of positive examples sampled per epoch to one-fourth of the original parameter according to the increase in batch size. In this task, we also adopt a two-stage training approach. The strictly clean in-domain training data for the first stage is obtained using a GMM model with a probability setting of 0.8.

\section{Related work}
\label{sec:relate}
\subsection{Noise Mitigation in Cross-Modal Matching}
The cross-modal matching task~\citep{SCAN, PVSE, VSRN, DAA, SGRAF} serves as a fundamental component in multimodal learning.
However, the inherent difference in information density between these modalities leads to high annotation costs and inconsistent annotation quality, rendering cross-modal tasks particularly vulnerable to label noise.
Some approaches explicitly identify and correct noisy samples through cross-prediction between concurrently trained dual models~\citep{NCR,BiCro,tripartite}, while others~\citep{NPC,DECL} implicitly estimate the probability of sample noise, reducing its training impact by adjusting the loss function. NCR~\citep{NCR} employs the memorization capacity of its counterpart model for simple clean samples to rectify the output results.
BiCro~\citep{BiCro} utilizes the consistency of similarity score distributions from a Siamese model ensemble on noisy data, alongside anchors modeled on the loss distribution via a Beta-Mixture-Model (BMM), to filter out noisy samples.
NPC~\citep{NPC}, deviating from the dual-model training schemes, introduces a two-stage single-model training approach to reduce the overhead by requiring only one model, while still needing two backward passes. Specifically, the first stage estimates the impact of potentially noisy samples on model performance by constructing a high-quality clean sample bank; the second stage then utilizes these estimates to reweight the loss function.
However, current methods for distinguishing clean from noisy samples rely on numerous hyperparameters that are closely linked to dataset size and model capacity. This dependency not only limits their adaptability to various downstream tasks but also makes them challenging to deploy in real-world applications.

\subsection{Noise Mitigation in Image Classification}

Image classification is vulnerable to training data noise, due to varied noise types and strong model memorization. 
Noise in datasets manifests in two primary forms: synthetic alterations and those arising from real-world scenarios. The former typically involves shuffling the labels of a subset of the data or retaining the labels while introducing corresponding category images from external datasets. The latter entails substituting images for a random selection of data points with those sourced from image search engines.
Existing approaches are categorized based on their operational focus: loss correction~\citep{loss_correction1, loss_correction2, loss_correction3_transmatrix,loss_correction4_transmatrix,loss_correction5_transmatrix,loss_correction6_transmatrix,robust_loss1,robust_loss2,robust_loss3,robust_loss4} and sample selection~\citep{pnp, pls, Jo-SRC, dividemix,dsos}. Loss correction methods typically incorporate a regularization term into the loss function, implicitly reweighting clean and noisy samples within the loss. Sample selection strategies, in contrast, explicitly differentiate between clean and noisy samples, applying distinct processing to each category during loss computation.
Representative for the loss correction category, \citep{loss_correction2} aims to generalize ordinary Cross-Entropy loss and MAE loss by setting the loss threshold to iid and ood noisy samples.
DivideMix~\citep{dividemix}  concurrently trains two networks, each utilizing the data partitioning from the other network to distinguish between clean and noisy samples based on loss values, thereby mitigating the influence of confirmation bias inherent within each network.
PNP~\citep{pnp} framework employs a unified predictive network to estimate the in-distribution (iid), out-of-distribution (ood), and clean probabilities for a given sample. Co-training trained on a sample that has a lower loss, and with the different predictions by its siamese network.


%

\subsection{Noise Mitigation in Image Retrieval}
Although image retrieval tasks focus on pairwise relationships, the noise predominantly originates from image categorization errors. Analogous to image classification tasks, this can be bifurcated into in-domain~\citep{retrieval_first} and open-set noise~\citep{retrieval_prism}. In terms of task configuration, noise retrieval typically operates at the category level, treating images within the same category as positive instances.
PRISM~\citep{retrieval_prism} tries to find noisy image samples by finding the outliers score in the whole similarity matrix from the same category. The generalization ability of the image feature is ensured by a broader query bank restored multi-view of it.
TITAN~\citep{retrieval_titan} utilizes prototypes to be representative of the anchor of the clean and noisy samples and then generates synthetic samples by a combination of prototypes for substitution of noisy samples.
T-SINT~\citep{retrieval_tisnt} utilizes more negative samples by the interaction between noisy samples and negative samples that belong to another category.

\section{Proofs}
\subsection{Proof of High-dimensional Orthogonality}
\label{sec:ortho_proof}
Suppose $u, v \in \mathbbm{R}^d$ are any two random vectors. The cosine similarity $\cos(u,v) \sim \mathcal{N}(0, d^{-1})$. The probability that $\cos(u,v)$ is within a specific range $\left [-a,a  \right ] $ is denoted as:
\begin{equation}
P(-a \leq \cos(u,v) \leq a) = \Phi\left(\frac{a}{\varsigma}\right) - \Phi\left(\frac{-a}{\varsigma}\right),
\label{eq:prob_com}
\end{equation}
where $\Phi$ represents the CDF of the standard normal distribution, and $\varsigma = \frac{1}{\sqrt{d}}$ is the standard deviation of the cosine similarity.
When $d=1024$ and $a = 0.1$, there are
\begin{equation}
\varsigma = \frac{1}{\sqrt{1024}}=\frac{1}{32},
\label{eq:deviation}
\end{equation}
and
\begin{multline}
P(-0.1 \leq \cos(u,v) \leq 0.1) = \\
\Phi\left(\frac{0.1}{1/32}\right) - \Phi\left(\frac{-0.1}{1/32}\right) \approx 0.9986.
\label{eq:results_multline}
\end{multline}


\subsection{Proof of Theorem 2.1}
\label{sec:proof1}
In the Section.~\ref{sec:cone effect}, we propose that Theorem 1 about the relative relationship of pairs in the original entire space, will not change after transmitting to the narrow cone space of the trained model, and there is always a boundary $\beta$ concentrated on most random vectors. 

To prove this Theorem, we first introduce a useful lemma of monotonicity of cosine similarity proposed by \citet{mindthegap}, indicating that the cosine similarity between two vectors increases with a high probability after one feedforward computation consisting of a linear transformation and ReLU computation.

\begin{lemma} 
\label{sec:lemma1}
Suppose $u,v \in \mathbbm{R}^{d_{in}}$ are any two fixed vectors such that $\left \| u \right \| =  r \left \| v \right \| $ for some $r>0$, $\mathbf{W} \in \mathbbm{R}^{d_{out} \times d_{in}}$ is a random weight matrix where each element $\mathbf{W}_{k,l} \sim \mathcal N(0,d_{out}^{-1})$ for $k \in \left [ d_{out}  \right ] $, $l \in \left [ d_{in}  \right ] $, and $\mathbf{b} \in \mathbbm{R}^{d_{out}}$ is a random bias vector such that $\mathbf{b}_k \sim \mathcal N(0,d^{-1}_{out})$ for $k \in \left [ d_{out}  \right ]$. If $\cos(u,v) <(\frac{1}{2}(r+\frac{1}{r}))^{-1}$, then the following holds with probability at least $1-O(1/d_{out})$.

\begin{equation}
\cos(\sigma(\mathbf{W}u+\mathbf{b}),\sigma(\mathbf{W}v+\mathbf{b})) > \cos(u,v).
\label{eq:lemma1}
\end{equation}
\end{lemma} 

\begin{proof}[Proof of Theorem.~\ref{thm:1}] 
Let $\mathbbm{R}^{d_{in}}$ be the original space before being transmitted in a neural network. Suppose $u,v \in \mathbbm{R}^{d_{in}}$ are any two random vectors with $\cos(u,v) \approx 0$. $u_c,v_c \in \mathbbm{R}^{d_{in}}$ is a pair of clean vectors with $\cos(u_c,v_c) > 0$, while $u_n,v_n \in \mathbbm{R}^{d_{in}}$ is a noisy pair with $\cos(u_n,v_n) < 0$. Given a Neural Network $F(x) = {f_t(f_{t-1}(\dots f_2(f_1(x))))} \in \mathbbm{R}^{d_{out}}$ with $t$ layers. $f_i(x)=\sigma_i(\mathbf{W}_ix+\mathbf{b}_i)$ denotes $i^{th}$ layer, where $\sigma(\cdot)$ indicates activation function. $\mathbf{W}_i \in \mathbbm{R}^{d^i_{out} \times d^i_{in}}$ is a random weight matrix where each element $\mathbf{W}_i^{k,l} \sim \mathcal N(0,1/d_{out}^{i})$ for $k \in \left [ d^i_{out}  \right ] $, $l \in \left [ d^i_{in}  \right ] $, and $\mathbf{b}_i \in \mathbbm{R}^{d^i_{out}}$ is a random bias vector such that $\mathbf {b}^k_i \sim \mathcal N(0,1/d^{i}_{out})$ for $k \in \left [ d^i_{out}  \right ]$. We would like to prove that there is always a boundary $\beta$, satisfying:
\begin{multline}
\cos(F(u_n),F(v_n)) < \cos(F(u),F(v)) \\
\approx \beta < \cos(F(u_c),F(v_c)),
\label{eq:theo1_rep}
\end{multline}
which is equivalent to proving,
\begin{multline}
\cos(f_i(u_n),f_i(v_n)) < \cos(f_i(u),f_i(v)) \\
\approx \beta_i < \cos(f_i(u_c),f_i(v_c)),
\label{eq:theo1_rep_f}
\end{multline}
where $\beta_i$ is the boundary of $i^{th}$ layer.

We first consider the cosine similarity between $u$ and $v$ as:
\begin{equation}
\cos(u, v) = \frac{u \cdot v}{\|u\|\|v\|}.
\label{eq:cos_tv}
\end{equation}
After a linear transformation of $i^{th}$ layer, the cosine similarity of $\cos(\mathbf{W}_iu+\mathbf{b}_i, \mathbf{W}_iv+\mathbf{b}_i)$ denotes:
\begin{multline}
\cos(\mathbf{W}_iu+\mathbf{b}_i,\mathbf{W}_iv+\mathbf{b}_i) = \\
\frac{(\mathbf{W}_iu+\mathbf{b}_i) \cdot (\mathbf{W}_iv+\mathbf{b}_i)}{\|\mathbf{W}_iu+\mathbf{b}_i\|\|\mathbf{W}_iv+\mathbf{b}_i\|}.
\label{eq:cos_trans_tv}
\end{multline}
Since $\mathbf{b}_i$ has a mean of zero and is independent from $\mathbf{W}_i u$ and $\mathbf{W}_i v$, the expectation of $\mathbf{b}_i$ and $(\mathbf{W}_i u+\mathbf{b}_i)\cdot \mathbf{W}_i v+\mathbf{b}_i)$ can be signified as:
\begin{equation}
\mathbb{E}\left [ \mathbf{b}_i \right ]  = 0,
\label{eq:ex_b}
\end{equation}
\begin{multline}
\mathbb{E}\left [(\mathbf{W}_i u+\mathbf{b}_i)\cdot (\mathbf{W}_i v+\mathbf{b}_i)\right ] \\
= \mathbb{E}\left [(\mathbf{W}_i u\cdot \mathbf{W}_i v)\right ]\\
=\sum^n_{i=1}\sum^n_{i=1}\frac{1}{d^i_{out}}u_kv_k=\frac{1}{d^i_{out}}(u\cdot v).
\label{eq:ex_dot}
\end{multline}

Additionally, we have
\begin{equation}
\|\mathbf{W}_iu+\mathbf{b}_i\|^2=\mathbf{W}_iu \cdot \mathbf{W}_iu + 2\mathbf{W}_iu \cdot \mathbf{b}_i + \mathbf{b}_i \cdot \mathbf{b}_i.
\label{eq:mode}
\end{equation}

Due to $\mathbbm {b}^k \sim \mathcal N(0,1/d^{i}_{out})$, as $d^{i}_{out}$ increases, the term of $2\mathbf{W}_iu \cdot \mathbf{b}_i$ and $\mathbf{b}_i \cdot \mathbf{b}_i$ become negligible, which implies
\begin{equation}
\|\mathbf{W}_iu+\mathbf{b}_i\|^2 \approx \mathbf{W}_iu \cdot \mathbf{W}_iu=\sum_{i=1}^n(\mathbf{W}_iu)^2.
\label{eq:mode2}
\end{equation}
Therefore, the expectation of $\|\mathbf{W}_iu+\mathbf{b}_i\|^2$ is approximate to
\begin{equation}
\mathbb{E}\left [ \|\mathbf{W}_iu\|^2 \right ] = \sum_{k=1}^n u_k^2 \frac{1}{d^i_{out}} = \frac{\|u\|}{d^i_{out}},
\label{eq:mode3_1}
\end{equation}
and 

\begin{equation}
\begin{aligned}
&\cos(\mathbf{W}_iu + \mathbf{b}_i, \mathbf{W}_iv + \mathbf{b}_i) \\
&\approx \frac{\mathbb{E}[\mathbf{W}_iu \cdot \mathbf{W}_iv]}{\sqrt{\mathbb{E}[\|\mathbf{W}_iu + \mathbf{b}_i\|^2] \mathbb{E}[\|\mathbf{W}_iv + \mathbf{b}_i\|^2]}} \\
&= \frac{\frac{1}{d^i_{out}} (u \cdot v)}{\sqrt{\frac{1}{d^i_{out}} \|u\|^2 \cdot \frac{1}{d^i_{out}} \|v\|^2}} \\
&= \cos(u, v).
\end{aligned}
\label{eq:equal}
\end{equation}

Based on Eq.~\ref{eq:equal}, with $\cos(u_n, v_n) < \cos(u, v) \approx 0 < \cos(u_c, v_c)$, there are
\begin{equation}
\begin{aligned}
&\cos(\mathbf{W}_i u_n+\mathbf{b}_i, \mathbf{W}_i v_n+\mathbf{b}_i) \\ 
&< \cos(\mathbf{W}_i u + \mathbf{b}_i, \mathbf{W}_i v + \mathbf{b}_i) \\
&< \cos(\mathbf{W}_i u_c + \mathbf{b}_i, \mathbf{W}_i v_c + \mathbf{b}_i).
\end{aligned}
\label{eq:mode3_2}
\end{equation}
Since the activation function $\sigma$ is a monotonically increasing function, it follows 

\begin{equation}
\begin{aligned}
&\cos(f_i(u_n),f_i(v_n)) < \cos(f_i(u),f_i(v)) \\
&< \cos(f_i(u_c),f_i(v_c)).
\end{aligned}
\label{eq:mode3_3}
\end{equation}

Due to Lemma.~\ref{sec:lemma1}, $\cos(f_i(u),f_i(v))$ will be increase with the transmitting layers, and $\cos(f_i(u),f_i(v))$ will always be a $\beta_i > 0$, to satisfy:
\begin{equation}
\begin{aligned}
&\cos(f_i(u_n),f_i(v_n)) < \cos(f_i(u),f_i(v)) \approx \beta_i \\
&< \cos(f_i(u_c),f_i(v_c)).
\label{eq:theo1_rep2}
\end{aligned}
\end{equation}

After transmitting each layer, Eq.~\ref{eq:theo1_rep2} is always satisfied. When transmitting a neural network with $t$ layers, we have
\begin{equation}
\begin{aligned}
&\cos(F(u_n),F(v_n)) < \cos(F(u),F(v)) \approx \beta \\
&< \cos(F(u_c),F(v_c)).
\label{eq:theo1_rep_last}
\end{aligned}
\end{equation}
\end{proof}

\begin{figure*}[ht!]
    \centering
    \small 
    \subfloat[parameters]{\label{gaussian:a}
        \includegraphics[width=0.24\linewidth,trim=0 0 0 10,clip]{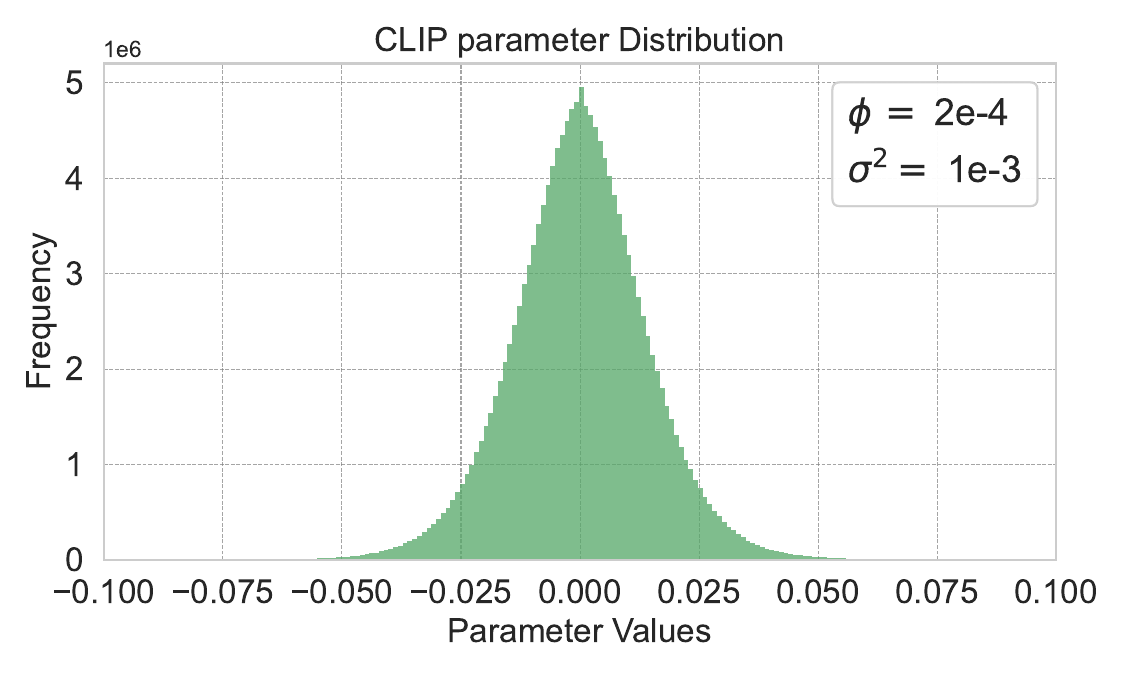}
    }
    \subfloat[128-th dim of image]
    {\label{gaussian:b}
        \includegraphics[width=0.24\linewidth,trim=0 0 0 10,clip]{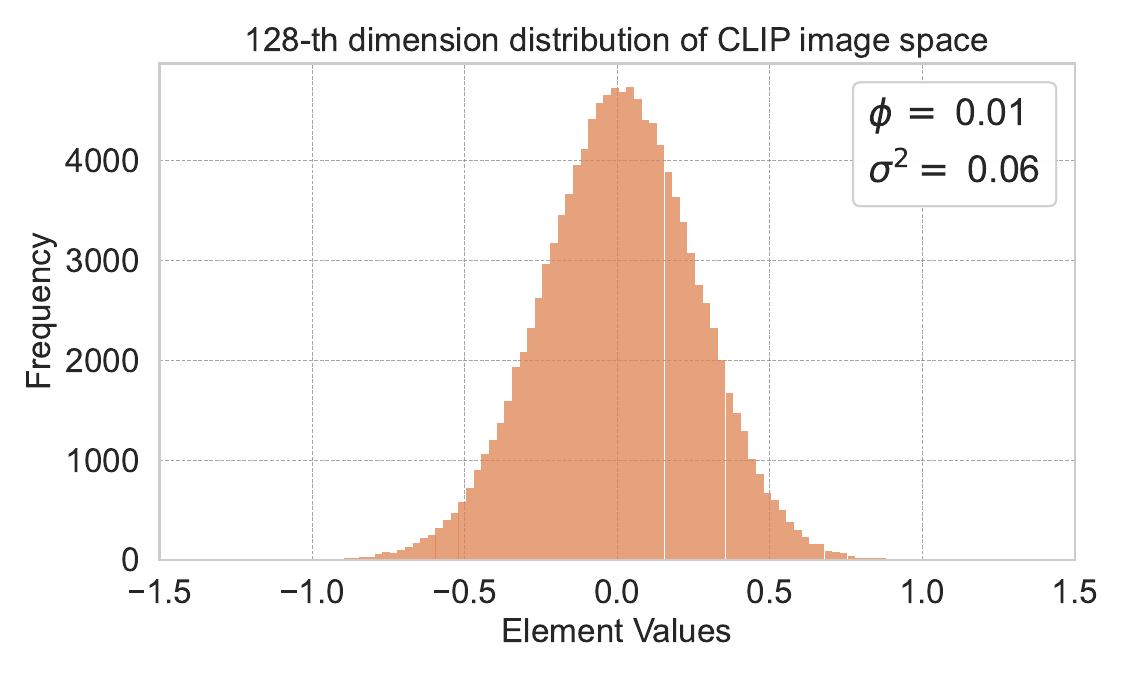}
    }
    \subfloat[256-th dim of image]{\label{gaussian:c}
        \includegraphics[width=0.24\linewidth,trim=0 0 0 10,clip]{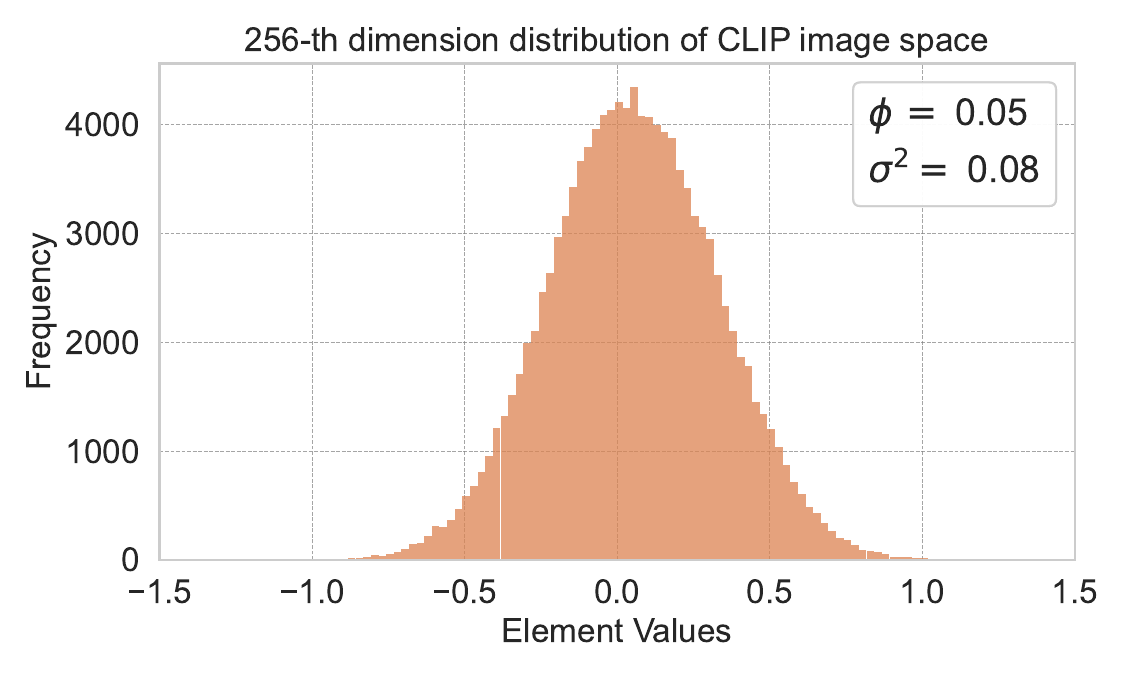}
    }
    \subfloat[512-th dim of image]{\label{gaussian:d}
        \includegraphics[width=0.24\linewidth,trim=0 0 0 10,clip]{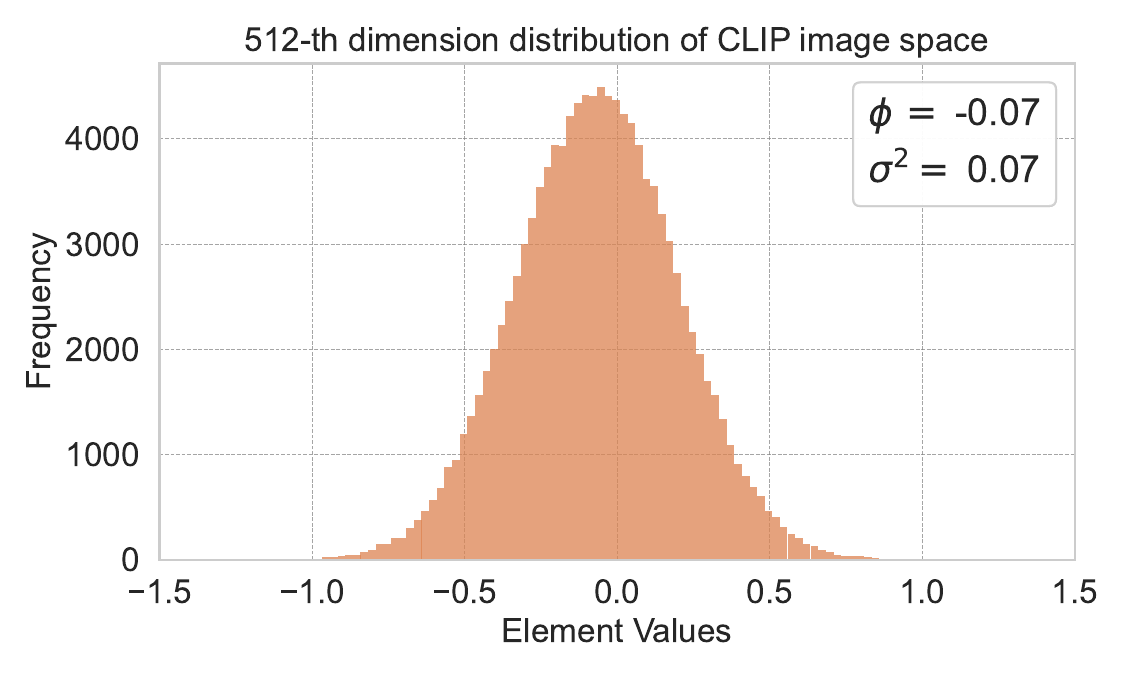}
    }
    \hfill
    \subfloat[128-th dim of text]{\label{gaussian:f}
        \includegraphics[width=0.24\linewidth,trim=0 0 0 10,clip]{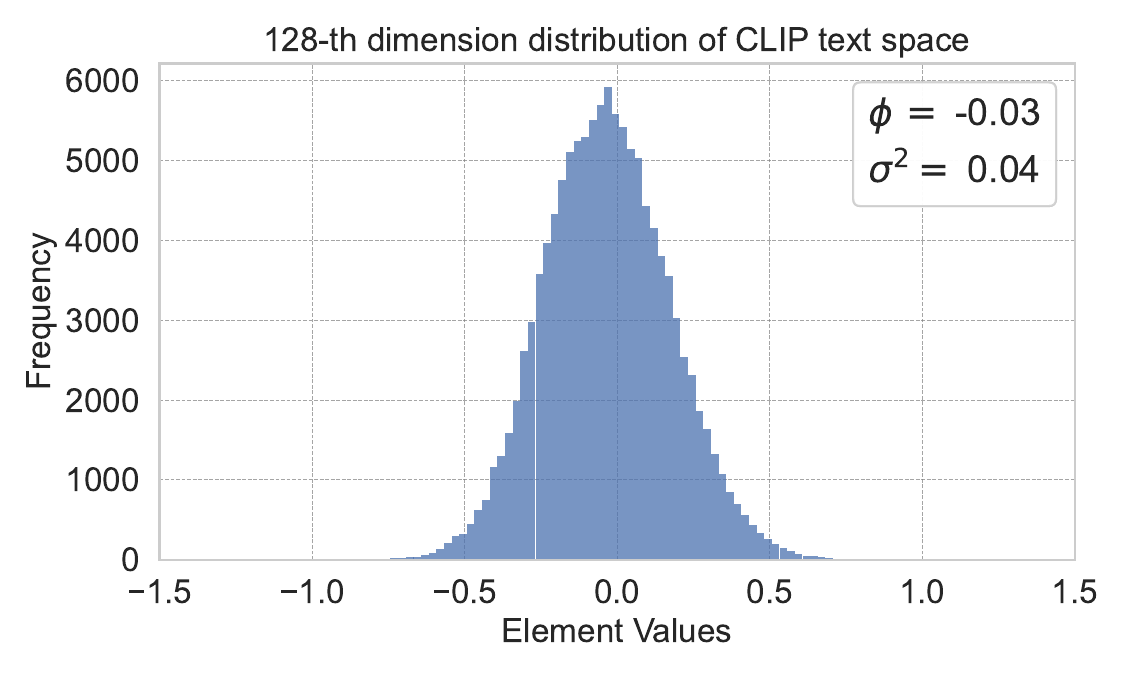}
    }
    \subfloat[256-th dim of text]
    {\label{gaussian:g}
        \includegraphics[width=0.24\linewidth,trim=0 0 0 10,clip]{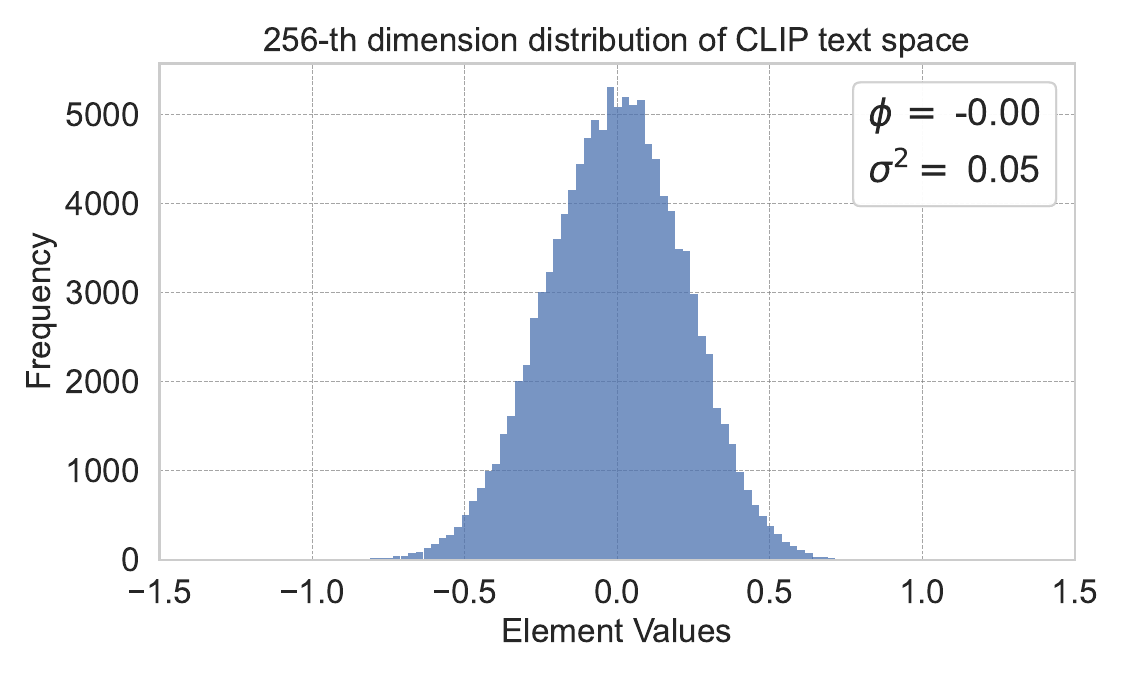}
    }
    \subfloat[512-th dim of text]{\label{gaussian:h}
        \includegraphics[width=0.24\linewidth,trim=0 0 0 10,clip]{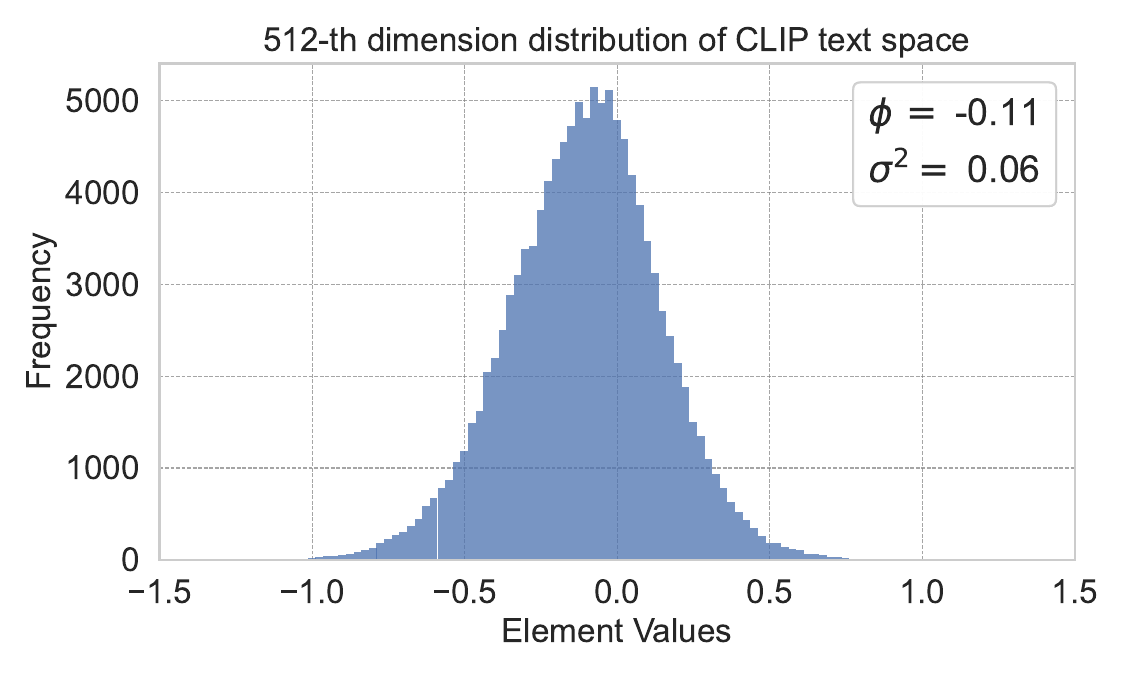}
    }
    \caption{The illustrations of several distributions on CC120K. (a) The parameter distribution. (b-d) The distribution of image features for the 128th, 256th, and 512th dimensions. (e-g) The distribution of text features for the 128th, 256th, and 512th dimensions.} 
    
    
\label{fig:gaussian}
\vspace{-4mm}
\end{figure*}

\subsection{Proof of Orthogonality Validity in Cone Space}
Although we have demonstrated in Appendix.~\ref{sec:ortho_proof} that in the original high-dimensional space, the cosine similarity between two randomly selected vectors—each dimension following a Gaussian distribution—typically converges near the orthogonal boundary, this property may not necessarily extend to the subspace of the shared embedding space maintained by the trained models. Specifically, for real image-text pairs, the subspace may deviate from the orthogonal characteristics observed in the original space. Thus, it is essential to investigate whether the orthogonality property holds within the cone space for the image-text subdomain post-training. 

To explore this, we first analyze the distribution of several dimensions of image and text features from the CC120K dataset, as illustrated in Figure.~\ref{fig:gaussian}. The results reveal that all vector dimensions, including trained parameters, exhibit a Gaussian distribution with near-zero means. 
This phenomenon arises from the convergence properties in two-layer networks that adhere to the central limit theorem~\citep{central_limit}, which, through the application of mean-field analysis, has been extended to neural network architectures~\citep{meanfield}. Consequently, the model parameters exhibit an overall Gaussian distribution, leading to Gaussian features. If the dimensions of the trained embedding space follow Gaussian distributions, the process of selecting random vectors within this space would be analogous to that of the original space, thereby preserving the orthogonality property. Here, we present the following theorem: The output features of large-scale models tend to Gaussian distribution. The detailed theorem and proof are provided below.

\begin{theorem}[Output features tends to Gaussian]\label{thm:2}
Given a Neural Network $F(x) = \{f_t(f_{t-1}(\dots f_2(f_1(x))))\} \in \mathbb{R}^{d_{out}}$ with $t$ layers. $f_l(x) = \phi_l(\mathbf{W}_l x + \mathbf{b}_l)$ denotes the $l^{th}$ layer, where $\phi(\cdot)$ indicates the activation function, and the final layer $f_t(x) = \mathbf{W}_t x + \mathbf{b}_t$ is a fully-connected layer without an activation function for common space projection. Let $x^k \in \mathbb{R}^{d_{in}^k}$ be the sample feature that will be transmitted into the $k^{th}$ layer, where $x^1$ denotes the original feature with an unknown distribution $x^1 \sim (\mu_x, \sigma_x^2)$. $\mathbf{W}_k \in \mathbb{R}^{d^k_{out} \times d^k_{in}}$ is a random weight matrix where each element $w^{k}_{ij} \sim \mathcal{N}(0, \sigma_w^2)$ for $i \in [d^k_{out}]$, $j \in [d^k_{in}]$, and $\mathbf{b}_k \in \mathbb{R}^{d^k_{out}}$ is a bias vector such that $b_i^k \sim \mathcal{N}(0, \sigma_w^2)$ for $i \in [d^k_{out}]$. In such a neural network, linear layers lead features $x$ gradually to a Gaussian distribution from any initial distribution, and as $|d_{in}|$ is sufficiently large, $F(x) \sim \mathcal{N}(0, \sigma^2)$.
\end{theorem}

\begin{proof}[Proof of Theorem.~\ref{thm:2}] 
For the $k^{th}$ layer ($k \in [t]$), we first calculate the expectation and variance of the linear combination $\sum_{j=1}^{d_{in}^k} w_{ij}^k x_j^k$. For the expectation, since $w_{ij}^k$ and $x_j^k$ are independent and $\mathbf{w}^{k}_{ij} \sim \mathcal{N}(0, \frac{1}{d_{out}^{k}})$, we have: 
\begin{equation}
\begin{aligned}
&\mathbb{E}\left[\sum_{j=1}^{d_{in}^k} w_{ij}^k x_j^k\right] = \sum_{j=1}^{d_{in}^k} \mathbb{E}[w_{ij}^k] \mathbb{E}[x_j^k]\\
&= \sum_{j=1}^{d_{in}^k} (0 \times \mathbb{E}[x_j^k]) = 0.
\end{aligned}
\label{eq:expectation_wx}
\end{equation}
For variance, since $w_{ij}^k$ and $x_j^k$ are independent, we have:

\begin{equation}
\begin{aligned}
&\text{Var}\left(\sum_{j=1}^{d_{in}^k} w_{ij}^k x_j^k\right)
= \sum_{j=1}^{d_{in}^k} \text{Var}(w_{ij}^k x_j^k) \\
&=\sum_{j=1}^{d_{in}^k} \mathbb{E}\left[(w_{ij}^k)^2 (x_j^k)^2\right] \\
&= \sum_{j=1}^{d_{in}^k} \mathbb{E}\left[(w_{ij}^k)^2\right] \left( \text{Var}(x_j^k) + (\mathbb{E}[x_j^k])^2 \right) \\
&= \sum_{j=1}^{d_{in}^k} \sigma_{w^k}^2 \left(\sigma_{x^k}^2 + \mu_{x^k}^2\right)
= d_{in}^k \sigma_{w^k}^2 \left(\sigma_{x^k}^2 + \mu_{x^k}^2\right).
\end{aligned}
\label{eq:variance_wx_1}
\end{equation}

Since $w_{ij}^k$ are independently distributed Gaussian random variables, and $x_j^k$ has a known mean and variance, the sum of $w_{ij}^k x_j^k$ can be applied to a generalized Central Limit Theorem. We have 
\begin{equation}
\frac{\sum_{j=1}^{d_{in}^k} w_{ij}^k x_j^k - \mathbb{E}\left[\sum_{j=1}^{d_{in}^k} w_{ij}^k x_j^k\right]}{\sqrt{\text{Var}\left(\sum_{j=1}^{d_{in}^k} w_{ij}^k x_j^k\right)}} \xrightarrow{d} \mathcal{N}(0, 1),
\label{eq:variance_wx_2}
\end{equation}
which is equivalent to 
\begin{equation}
\frac{\sum_{j=1}^{d_{in}^k} w_{ij}^k x_j^k - 0}{\sqrt{d_{in}^k \sigma_{w^k}^2 (\sigma_{x^k}^2 + \mu_{x^k}^2)}} \xrightarrow{d} \mathcal{N}(0, 1).
\end{equation}
Therefore, 
\begin{equation}
\sum_{j=1}^{d_{in}^k} w_{ij}^k x_j^k \xrightarrow{d} \mathcal{N}(0, d_{in}^k \sigma_{w^k}^2 (\sigma_{x^k}^2 + \mu_{x^k}^2)).
\end{equation}
Due to $b^k \sim \mathcal{N}(0, \sigma^2_b)$, we finally get 
\begin{equation}
\sum_{j=1}^{d_{in}^k} w_{ij}^k x_j^k + b_i^k \xrightarrow{d} \mathcal{N}\left(0, d_{in}^k \sigma_{w^k}^2 (\sigma_{x^k}^2 + \mu_{x^k}^2) + \sigma_b^2\right).
\end{equation}
Although activation functions truncate the Gaussian distribution after each linear layer, the samples still gradually approach a Gaussian distribution from the initial unknown distribution as they pass through the layers. Furthermore, because there is a fully connected layer without an activation function before mapping to the final common space, the final feature distribution will approximate a Gaussian distribution, as follows:
\begin{equation}
F(x) \sim \mathcal{N}(0,d^t_{in}\sigma^2_{w^t}(\sigma^2_{x^t} + \mu^2_{x^t})+\sigma^2_b).
\end{equation}
\end{proof}

\section{Additional Experimental Results}
\label{sec:add_ex}

\subsection{Full Quantitative Results on Different Downstream Tasks}
\label{sec:append_qualititive}
\gjy{
\textbf{Results on Flickr30K.}
To further demonstrate the generalization ability of OSA, we evaluate on the Flickr30K dataset and compare with several anti-noise methods, including NCR~\citep{NCR}, DECL~\citep{DECL}, BiCro~\citep{BiCro}, and NPC~\citep{NPC}. The results are presented in Table.~\ref{table:f30k}. It is evident that OSA consistently outperforms all models on the R@1 metric. Notably, compared with the baseline CLIP, training with OSA at a 60\% noise ratio achieves 20.9\% R@1 improvement for i2t and a 22.3\% R@1 improvement in t2i, further indicating the effectiveness of OSA on noise mitigation. Additionally, OSA demonstrates similar noise robustness on the Flickr30K dataset as observed on MSCOCO, with only 1.4\% R@1 drop on i2t and 1.2\% R@1 drop on t2i ranging from 0\% noise to 60\% noise, while all of the other anti-noise approaches hardly resist the detriment from high-ratio noise. All of these results demonstrate the effectiveness and robustness of OSA on anti-noise tasks.

\textbf{Results on CC120K and CC3M.}
To further verify the reliability of OSA in real scenarios, we conduct evaluations on the subset and the original large-scale, real-world noisy dataset, CC120K and CC3M. The noise ratio is estimated as 3\%-20\%.
CC120K serves as a subset of CC3M, which indicates the lower implementation difficulties.
As NPC requires much more memory to construct the memory bank for each sample, we are not able to reproduce its result on CC3M.
The results in Table.~\ref{table:cc120k_3m} indicates that OSA outperforms the current state-of-the-art method NPC, even in real-world domains.
And is much easier to use in a real-world large data scale scenario.
This demonstrates the feasibility and generality of OSA even in practical training scenarios.


}
\begin{table*}[t]
\caption{Comparison on noisy Flickr30K.}
  \renewcommand{\arraystretch}{0.9}
  \setlength{\tabcolsep}{0.9mm}
  \centering
  \scalebox{0.85}{
\begin{tabular}{c|c|ccc|ccc|c|ccc|ccc}
\toprule
\multirow{2}{*}{Method} & \multirow{2}{*}{\small{Noise ratio}} & \multicolumn{3}{c|}{i2t} & \multicolumn{3}{c|}{t2i} & \multirow{2}{*}{\small{Noise ratio}} & \multicolumn{3}{c|}{i2t} & \multicolumn{3}{c}{t2i}\\
                       &   & R@1     & R@5    & R@10    & R@1     & R@5    & R@10 &   & R@1     & R@5    & R@10    & R@1     & R@5    & R@10   \\ 
                       \cmidrule(r){1-15}
    NCR & \multirow{6}{*}{0\%}  & 77.3          & 94.0          & 97.5          & 59.6          & 84.4          & 89.9   & \multirow{6}{*}{20\%}  & 73.5          & 93.2          & 96.6          & 56.9          & 82.4          & 88.5       \\
    DECL &  & 79.8          & 94.9          & 97.4          & 59.5          & 83.9          & 89.5   &  & 77.5          & 93.8          & 97.0          & 56.1          & 81.8          & 88.5       \\
    BiCro & & 81.7          & 95.3          & 98.4          & 61.6          & 85.6          & 90.8   &  & 78.1          & 94.4          & 97.5          & 60.4          & 84.4          & 89.9       \\
    NPC  & & 87.9 & \textbf{98.1} & \textbf{99.4} & 75.0 & \textbf{93.7} & \textbf{97.2} & & 87.3 & 97.5 & 98.8 & 72.9 & 92.1 & 95.8 \\
    \cmidrule(r){1-1} \cmidrule(r){3-8} \cmidrule(r){10-15}
    CLIP & & 86.2          & 97.6          & 99.2          & 72.9          & 92.3          & 96.0   & & 82.3          & 95.5          & 98.3          & 66.0          & 88.5          & 93.5       \\
    \textbf{+OSA} & & \textbf{88.6} & 97.7 & 99.3 & \textbf{75.6} & 93.6 & 96.8 & & \textbf{88.9} & \textbf{97.7} & \textbf{99.1} & \textbf{75.6} & \textbf{93.3} & \textbf{96.9} \\
                      \cmidrule(r){1-15}
    NCR & \multirow{6}{*}{40\%}   & 68.1          & 89.6          & 94.8          & 51.4          & 78.4          & 84.8   & \multirow{6}{*}{60\%}  & 13.9          & 37.7          & 50.5          & 11.0          & 30.1          & 41.4     \\
    DECL &  & 72.7          & 92.3          & 95.4          & 53.4          & 79.4          & 86.4   & & 65.2          & 88.4          & 94.0          & 46.8          & 74.0          & 82.2      \\
    BiCro & & 74.6          & 92.7          & 96.2          & 55.5          & 81.1          & 87.4   & & 67.6          & 90.8          & 94.4          & 51.2          & 77.6          & 84.7      \\
    NPC & & 85.6 & 97.5 & 98.4 & 71.3 & 91.3 & 95.3 & & 83.0 & 95.9 & 98.6 & 68.1 & 89.6 & 94.2 \\
    \cmidrule(r){1-1} \cmidrule(r){3-8} \cmidrule(r){10-15}
    CLIP &  & 76.2          & 93.3          & 96.5          & 59.4          & 85.0          & 90.9   & & 66.3          & 87.3          & 93.0            & 52.1          & 78.8          & 87.4      \\
    \textbf{+OSA} & & \textbf{87.3} & \textbf{97.6} & \textbf{99.3} & \textbf{74.2} & \textbf{93.1} & \textbf{96.7} & & \textbf{87.2} & \textbf{98.1} & \textbf{99.6} & \textbf{74.4} & \textbf{92.9} & \textbf{96.4}  \\                    
                      \cmidrule(r){1-15}

\end{tabular}}
\label{table:f30k}
\end{table*}


\begin{table*}[h]
    \centering
    \caption{Comparison on real-world noisy dataset CC120K and CC3M.}
    \renewcommand{\arraystretch}{1.15}
    \setlength{\tabcolsep}{1mm}
    \scalebox{0.8}{
    \begin{tabular}{l|ccc|ccc|ccc|ccc}
    \hline
    \multirow{3}{*}{Method} & \multicolumn{6}{c|}{CC120K} & \multicolumn{6}{c}{CC3M} \\
    \cline{2-13}
     & \multicolumn{3}{c|}{\textbf{i2t}} & \multicolumn{3}{c}{\textbf{t2i}}  & \multicolumn{3}{|c|}{\textbf{i2t}} & \multicolumn{3}{c}{\textbf{t2i}}\\
     & \textbf{R@1} & \textbf{R@5} & \textbf{R@10} & \textbf{R@1} & \textbf{R@5} & \textbf{R@10} 
     & \textbf{R@1} & \textbf{R@5} & \textbf{R@10} & \textbf{R@1} & \textbf{R@5} & \textbf{R@10}  \\
    \hline
    NPC & 71.1 & 92.0 & \textbf{96.2} & 73.0 & 90.5 & \textbf{94.8} 
    & \multicolumn{6}{c}{\textit{Out Of Memory}} \\
    CLIP & 68.8 & 87.0 & 92.9 & 67.8 & 86.4 & 90.9 
    & 42.41 & 66.70 & 75.56 & 42.45 & 67.83 & 76.46 \\
    \cline{1-13}
    \textbf{         +OSA} & \textbf{73.1} & \textbf{92.2} & 95.7 & \textbf{73.9} & \textbf{91.2} & 94.7 & \textbf{43.34} & \textbf{67.48} & \textbf{75.79} & \textbf{43.46} & \textbf{68.33} & \textbf{76.58} \\
    \hline
    \end{tabular}
    }
    \label{table:cc120k_3m}
\end{table*}

\subsection{Full Results on Different Scoring Function}
\label{sec:diff_scoring}
In Section~\ref{sec: scoring_main}, we introduce a scoring function designed to handle overlaps effectively and re-weight samples based on their cosine similarity. This section explores several scoring functions, including the Linear, Cosine, and High-degree functions. The functions are presented individually as follows:

\textbf{Linear Function.} 
\begin{equation}
w(x,y,\beta)=\left\{\begin{aligned}
& 0 &, & \tilde{s}_{x,y} \le 0\\
& \tilde{s}_{x,y} &, & \textit{otherwise} \hfill\\
\end{aligned}\right.
\label{eq:scoring_1}
\end{equation}

\textbf{Cosine Function.} 
\begin{equation}
w(x,y,\beta)=\left\{\begin{aligned}
& 0 &, & \tilde{s}_{x,y} \le 0\\
& \frac{\cos(\pi(\tilde{s}_{x,y} - 1)) + 1}{2} &, & \textit{otherwise} \hfill\\
\end{aligned}\right.
\label{eq:scoring_2}
\end{equation}

\textbf{High-degree Function.} 
\begin{equation}
w(x,y,\beta)=\left\{\begin{aligned}
& 0 &, & \tilde{s}_{x,y} \le 0\\
& -(\tilde{s}_{x,y})^2(\tilde{s}_{x,y}-1) &, & \textit{otherwise} \hfill\\
\end{aligned}\right.
\label{eq:scoring_3}
\end{equation}

The results are presented in the Table.~\ref{table:scoring_function} indicates that the High-Degree Function outperforms the others across all evaluation metrics. This superior performance can be attributed to the rapid gradient changes near the decision boundary, which align better with the tendency of orthogonal boundaries. As a result, we adopt the High-Degree Function as our scoring function.

\begin{table*}[ht]
  \caption{Comparison of different scoring functions.}
  \label{table:scoring_function}
  \renewcommand{\arraystretch}{0.8}
  \setlength{\tabcolsep}{1.2mm}
  \centering
  \scalebox{0.8}{
  \begin{tabular}{c|l|ccc|ccc|ccc|ccc}
    \toprule
    \multirow{3}{*}{Noise ratio} & \multirow{3}{*}{Method} & \multicolumn{6}{c|}{MS-COCO 1K} & \multicolumn{6}{c}{MS-COCO 5K} \\
     & & \multicolumn{3}{c}{i2t} & \multicolumn{3}{c|}{t2i}  & \multicolumn{3}{c}{i2t} & \multicolumn{3}{c}{t2i}\\
     &  &  \multicolumn{1}{c}{R@1} & \multicolumn{1}{c}{R@5} & \multicolumn{1}{c}{R@10} &
                                \multicolumn{1}{c}{R@1} & \multicolumn{1}{c}{R@5} & \multicolumn{1}{c|}{R@10} &  \multicolumn{1}{c}{R@1} & \multicolumn{1}{c}{R@5} & \multicolumn{1}{c}{R@10} &
                                \multicolumn{1}{c}{R@1} & \multicolumn{1}{c}{R@5} & \multicolumn{1}{c}{R@10}   
                                \\
    \cmidrule(r){1-14}
     \multirow{3}{*}{50\%}  & Linear &  80.4   &   96.2  &  98.6 & 67.8  & 91.6 & 96.4 &  64.0  & 85.5 & 91.9 & 47.9  & 74.6 & 83.8 \\
    & Cosine &  80.8   &  96.3  &  98.5 & 67.7  & 91.6 & 96.3 & 64.4  & 86.2 & 92.3 & 48.0  & 74.9 & 83.9   \\
    & High-Degree &  \textbf{81.4}   &  \textbf{96.5} &  \textbf{98.6} & \textbf{68.4}  & \textbf{92.0} & \textbf{96.6} &  \textbf{64.7}  & \textbf{86.8} & \textbf{92.4} & \textbf{48.6}  & \textbf{75.9} & \textbf{84.6}   \\
    
    \bottomrule
  \end{tabular}}
\end{table*}

\subsection{Full Results on OSA Transferability}
\label{sec:transfer}

\textbf{Results on transferring OSA to other noise mitigation methods.}
We apply OSA to give guidance to other noise mitigation methods to prove that our framework follows a plug-and-play style and can be easily fit into existing methods. From Table. \ref{table:apply2npc}, we observe that even integrated with the previous SOTA method, OSA can still give effective guidance to it, thereby the generalization ability of our method can be further proved.

\begin{table}[H]
\caption{The results of other methods employing OSA on MSCOCO 1K.}
\renewcommand{\arraystretch}{1.25}
\setlength{\tabcolsep}{1mm}
\centering
\scalebox{0.7}{
\begin{tabular}{c|l|ccc|ccc}
\hline
\multirow{2}{*}{Noise Ratio} & \multirow{2}{*}{Method}  & \multicolumn{3}{c|}{\textbf{i2t}} & \multicolumn{3}{c}{\textbf{t2i}} \\
  &  & \textbf{R@1} & \textbf{R@5} & \textbf{R@10} & \textbf{R@1} & \textbf{R@5} & \textbf{R@10} \\
\hline
\multirow{2}{*}{0\%}  & NPC           & 82.2 & \textbf{96.5} & \textbf{98.7} & 68.3 & \textbf{92.0} & \textbf{98.7} \\
     & \textbf{\qquad+OSA}       & \textbf{82.4} & 96.4 & 98.6 & \textbf{68.5} & 91.8 & \textbf{98.7} \\
\hline
\multirow{2}{*}{20\%} & NPC           & 79.9 & 95.9 & 98.4 & 66.3 & 90.5 & 98.4 \\
     & \textbf{\qquad+OSA}       & \textbf{81.2} & \textbf{96.0} & \textbf{98.6} & \textbf{66.9} & \textbf{91.2} & \textbf{98.6} \\
\hline
\multirow{2}{*}{50\%} & NPC           & 78.2 & 94.4 & 97.7 & 63.1 & 89.0 & 97.7 \\
     & \textbf{\qquad+OSA}       & \textbf{79.3} & \textbf{95.6} & \textbf{98.2} & \textbf{66.8} & \textbf{90.8} & \textbf{98.2} \\
\hline
\end{tabular}
}
\label{table:apply2npc}
\end{table}


\begin{table*}[t]
  \caption{The results of different estimator types on noisy MS-COCO.}
  \label{table:domain}
  \renewcommand{\arraystretch}{0.9}
  \setlength{\tabcolsep}{1mm}
  \centering
  \scalebox{0.825}{
  \begin{tabular}{c|l|ccc|ccc|ccc|ccc}
    \toprule
    \multirow{3}{*}{Noise ratio} & \multirow{3}{*}{Estimator} & \multicolumn{6}{c|}{MS-COCO 1K} & \multicolumn{6}{c}{MS-COCO 5K} \\
     & & \multicolumn{3}{c}{i2t} & \multicolumn{3}{c|}{t2i}  & \multicolumn{3}{c}{i2t} & \multicolumn{3}{c}{t2i}\\
     &  &  \multicolumn{1}{c}{R@1} & \multicolumn{1}{c}{R@5} & \multicolumn{1}{c}{R@10} &
                                \multicolumn{1}{c}{R@1} & \multicolumn{1}{c}{R@5} & \multicolumn{1}{c|}{R@10} &  \multicolumn{1}{c}{R@1} & \multicolumn{1}{c}{R@5} & \multicolumn{1}{c}{R@10} &
                                \multicolumn{1}{c}{R@1} & \multicolumn{1}{c}{R@5} & \multicolumn{1}{c}{R@10}   
                                \\
    \cmidrule(r){1-14}
     \multirow{4}{*}{0\%}   & None    &  80.1   &   95.7  &  98.2 & 67.1  & 91.4 & 96.6 &  62.9   &   84.9  &  91.6 & 46.5  & 73.8 & 82.9 \\
    & CLIP (w/o DA)           & \textbf{82.6} & \textbf{96.7} & 98.7 & 68.5 & \textbf{92.1} & \textbf{96.7} & \textbf{66.2} & \textbf{87.0} & \textbf{93.3} & 48.6 & 75.7 & \textbf{84.8} \\
    & ALIGN (w/o DA)&  81.9    &   96.7      &  98.7    & 68.9   & 92.2     & 96.9  &   64.8    &   86.6      &  92.7  & 49.0   & 75.9 & 84.7   \\
    & \textbf{CLIP (w DA)}            & 82.2 & 96.5 & 98.7 & \textbf{68.8} & 92.1 & 96.7 & 65.6 & 86.8 & 92.9 & \textbf{49.1} & \textbf{76.2} & \textbf{84.8}  \\
    
    \cmidrule(r){1-14}
     \multirow{4}{*}{20\%}   &  None     &   76.0   &   94.3  &  97.5 & 63.4  & 89.0 & 94.8 &   55.3   &   79.1  &  86.9 & 41.0  & 68.8 & 79.3 \\
    & CLIP (w/o DA)            & \textbf{81.8} & 96.1 & \textbf{98.7} & 68.2 & 91.9 & 96.5 & 64.8 & \textbf{86.6} & 92.3 & 48.3 & 75.4 & 84.1  \\
    & ALIGN (w/o DA) &  81.2    &   96.0      &  98.6    & 67.7   & 91.5     & 96.4  &   64.8    &   86.2      & 92.3  & 47.8   & 74.9 & 83.9   \\
    & \textbf{CLIP (w DA)}     &  81.6   &  \textbf{96.2}  &  98.5 & \textbf{68.9}  & \textbf{92.0} & \textbf{96.6} &  \textbf{65.8}  & 86.4 & \textbf{92.5} & \textbf{48.7}  & \textbf{76.1} & \textbf{84.5} \\
    
    \cmidrule(r){1-14}
     \multirow{4}{*}{50\%}   & None     &  73.9   &   93.0  &  97.2 & 60.1  & 87.3 & 94.0  &  54.1   &   78.5  &  86.6 & 39.7  & 67.2 & 77.5\\
    & CLIP (w/o DA)&  79.6    &   95.6      &  98.4    & 65.9   & 90.8     & 95.9  &   62.4    &   84.8      &  90.8  & 45.7   & 73.1 & 82.5   \\
     & ALIGN (w/o DA) &  80.4    &   95.6      &  98.3    & 66.0   & 90.5     & 95.8  &   62.0    &   84.9      &  91.8  & 45.7   & 73.2 & 82.5   \\
    & \textbf{CLIP (w DA)}     &  \textbf{81.4}   &  \textbf{96.5} &  \textbf{98.6} & \textbf{68.4}  & \textbf{92.0} & \textbf{96.6} &  \textbf{64.7}  & \textbf{86.8} & \textbf{92.4} & \textbf{48.6}  & \textbf{75.9} & \textbf{84.6} \\
    \bottomrule
  \end{tabular}}
\end{table*}

\textbf{Results for overall noise rank reliability.}
\label{sec:append_noise_rank}
Some anti-noise methods, like NPC, also employ loss re-weighting for optimization. To assess whether our method assigns relatively smaller weights to noise than these methods, we first analyze the weights generated by NPC and OSA. Due to differences in weight scales across methods, a direct comparison is unfair. We therefore adopt a ranking-based approach, sorting weights in descending order and calculating the Mean Noise Rank to unify the scale. This metric evaluates whether smaller weights are consistently assigned to noisy samples relative to clean ones. Our experiments use 2,000 randomly selected samples from the MSCOCO dataset under two noise conditions: 20\% noise (370 noisy samples) and 50\% noise (953 noisy samples). The theoretical optimal Mean Noise Ranks, where all noisy weights are ranked last, are 1815.5 and 1524.0, respectively. Results presented in Table.~\ref{table:weight_rank} show that OSA achieves a higher Mean Noise Rank compared to NPC, demonstrating greater accuracy in re-weighting. Moreover, OSA’s rankings are nearly optimal (20\% noise: 1809.1 for OSA vs. 1815.5 optimal; 50\% noise: 1520.7 for OSA vs. 1524.0 optimal). This near-perfect alignment indicates that OSA effectively places almost all noisy samples behind the clean ones.
\begin{table*}[h]
\centering
\caption{Mean Noise Rank Comparison between OSA and NPC.}
\label{table:weight_rank}
  \renewcommand{\arraystretch}{1.2}
  \setlength{\tabcolsep}{0.8mm}
  \scalebox{0.9}{
\begin{tabular}{c|c|c|c|c|c}
\hline
\textbf{Noise Ratio} & \textbf{Method} & \textbf{Mean Noise Rank}$\uparrow$ & \textbf{Optimal Rank} & \textbf{Noise Number} & \textbf{Sample Number} \\
\hline
\multirow{2}{*}{20\%} & NPC & 1641.3 & 1815.5 & 370 & 2,000 \\
 & OSA & \textbf{1809.1} & 1815.5 & 370 & 2,000 \\
\cmidrule(r){1-6}
\multirow{2}{*}{50\%} & NPC & 1456.2 & 1524.0 & 953 & 2,000 \\
 & OSA & \textbf{1520.7} & 1524.0 & 953 & 2,000 \\
\hline
\end{tabular}}
\end{table*}

\begin{table}[H]
    \centering
        \renewcommand{\arraystretch}{1.2}
        \centering
        \caption{Overhead Comparison.}
        \setlength{\tabcolsep}{3mm}
        \scalebox{0.95}{
        \begin{tabular}{c|c|c}
        \toprule
        Model & Time & Extra Cost\\
        \cmidrule(r){1-3}
        CLIP & 97 min &  0 min \\
        NPC & 323 min & 226 min \\
        OSA & 118 min & 21 min \\
        \bottomrule
        \end{tabular}}
        \label{table:time}
\end{table}

\textbf{Comparisons of extra cost.}
As mentioned in Sec. \ref{sec:intro}, existing noise mitigation methods tend to have extra steps before training the target model. Table. \ref{table:time} shows the comparisons between OSA and NPC on the extra time cost. As NPC needs extra time to coarsely filter some relative clean samples and use them to warm up the model, OSA just needs a forward step to calculate the noise probability, which leads to less extra time cost.

\section{SDM Visualization}
\label{sec:SDM_visualization}
We visualize some representative samples from our synthetic domain originating from COCO by using SDM. The results are shown in Figure.~\ref{fig:SDM}. We generate two styles of image based on the MSCOCO caption, and then use pre-trained multimodal models to calculate cosine similarity with the SDM-generated image and original caption.

\begin{figure*}[t]
    \centering
    \includegraphics[width=1\linewidth,trim=80 80 80 80,clip]{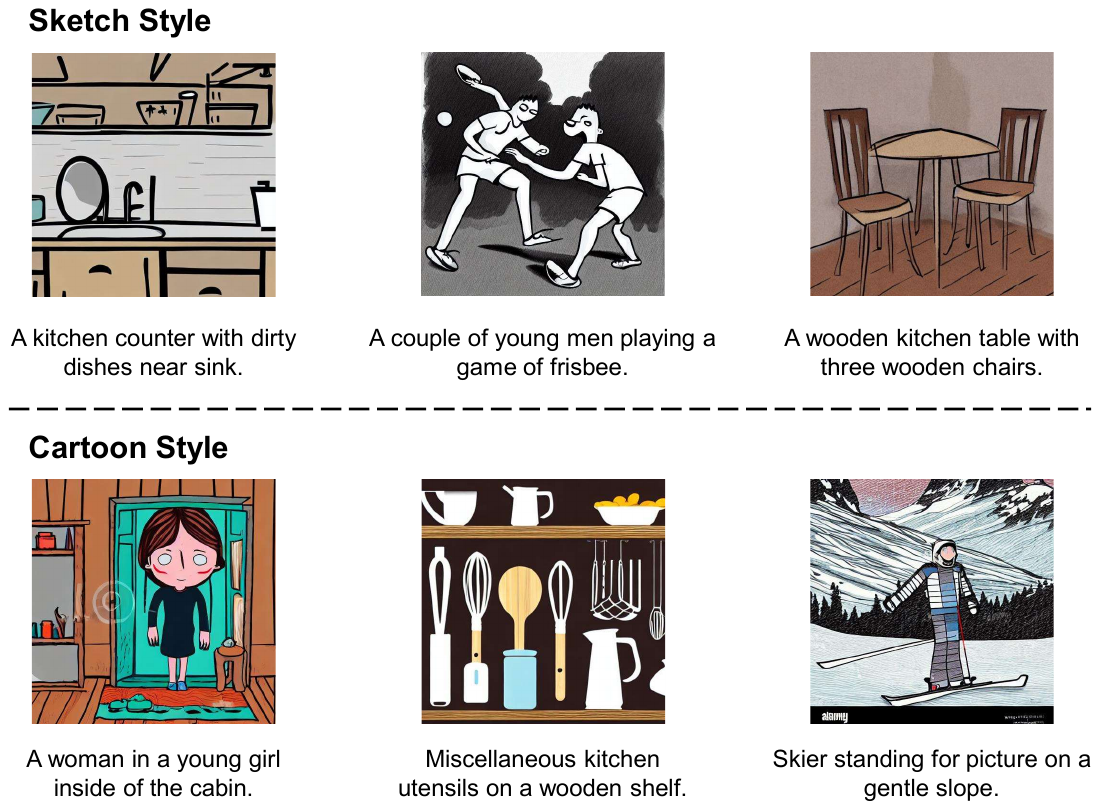}
    \caption{Examples of generated SDM dataset. The first row is in sketch style, while the second row is in cartoon style.}
    \label{fig:SDM}
\end{figure*}

\end{document}